\documentclass[conference]{IEEEtran}
\IEEEoverridecommandlockouts
\usepackage{subcaption}
\usepackage{amsmath,amssymb,amsfonts}
\usepackage[english, ruled, linesnumbered]{algorithm2e} 
\usepackage{graphicx}
\usepackage{textcomp}
\usepackage{xcolor}
\usepackage[utf8]{inputenc}
\usepackage{url}
\usepackage{multirow}
\usepackage{booktabs}
\usepackage{todonotes}
\usepackage{amsthm}
\usepackage{fdsymbol}
\usepackage{mathtools}
\usepackage{balance}
\usepackage{hyperref}
\hypersetup{
	colorlinks=true, 
	linktoc=all,     
	linkcolor={black},
	urlcolor={blue},
	anchorcolor=black, 
	citecolor=black,
}

\DeclarePairedDelimiter\floor{\lfloor}{\rfloor}

\SetKwComment{Comment}{$\triangleright$\ }{}

\SetCommentSty{mycommfont}

\newcounter{mythmcounter}
\newtheorem{theorem}{Theorem}[mythmcounter]
\newtheorem{corollary}{Corollary}[theorem]
\newtheorem{lemma}[theorem]{Lemma}
\newtheorem{conjecture}[theorem]{Conjecture}

\DeclareMathOperator*{\argmin}{arg\,min}
\DeclareMathOperator*{\argmax}{arg\,max}

\newcommand{\Resources}{\mathcal{R}}
\newcommand{\ResNum}{n}
\newcommand{\Tasks}{T}
\newcommand{\Cost}[2]{C_{{#1}} ({#2})}
\newcommand{\Costs}[2]{C_{{#1}} = \{{#2}\}}
\newcommand{\Mapping}[1]{A_{#1}}
\newcommand{\MappingT}[2]{A^{#2}_{#1}}
\newcommand{\Real}{\mathbb{R}_{\ge 0}}
\newcommand{\Natural}{\mathbb{N}}
\newcommand{\Makespan}{C_\text{max}}
\newcommand{\MakespanT}[1]{C^{#1}_\text{max}}
\newcommand{\Lower}[1]{L_{#1}}
\newcommand{\Upper}[1]{U_{#1}}
\newcommand{\LowV}{l}
\newcommand{\UpV}{u}

\def\BibTeX{{\rm B\kern-.05em{\sc i\kern-.025em b}\kern-.08em
    T\kern-.1667em\lower.7ex\hbox{E}\kern-.125emX}}
\begin{document}

\title{Optimal Task Assignment to Heterogeneous Federated Learning Devices}

\author{
\IEEEauthorblockN{La\'ercio L. Pilla}
\IEEEauthorblockA{\textit{Univ. Paris-Saclay, CNRS, Laboratoire de Recherche en Informatique (LRI)} \\
Orsay, France \\
pilla@lri.fr}
}

\maketitle

\begin{abstract}
Federated Learning provides new opportunities for training machine learning models while respecting data privacy.
This technique is based on heterogeneous devices that work together to iteratively train a model while never sharing their own data.
Given the synchronous nature of this training, the performance of Federated Learning systems is dictated by the slowest devices, also known as stragglers.
In this paper, we investigate the problem of minimizing the duration of Federated Learning rounds by controlling how much data each device uses for training.
We formulate this problem as a makespan minimization problem with identical, independent, and atomic tasks that have to be assigned to heterogeneous resources with non-decreasing cost functions while respecting lower and upper limits of tasks per resource.
Based on this formulation, we propose a polynomial-time algorithm named OLAR and prove that it provides optimal schedules.
We evaluate OLAR in an extensive experimental evaluation using simulation that includes comparisons to other algorithms from the state of the art and new extensions to them.
Our results indicate that OLAR provides optimal solutions with a small execution time.
They also show that the presence of lower and upper limits of tasks per resource erase any benefits that suboptimal heuristics could provide in terms of algorithm execution time.

\end{abstract}


\section{Introduction}


Federated Learning~(FL) is a recent machine learning technique focused on data privacy and security~\cite{mcmahan2017communication, bonawitz2019towards, lim2020federated}.
In a nutshell, FL involves a group of participating devices (mostly heterogeneous mobile devices) working together to iteratively train a machine learning model under the coordination of a central FL server.
The server starts the training by sending an initial model to the devices. 
Each device trains the model with its own local dataset --- which is never shared --- and sends the new model's weights to the server.
The server then aggregates all updates, averages them, and sends the new model weights to the devices.
This process is repeated for a fixed number of rounds, or until the model convergences to a target accuracy.
Overall, this training scheme is also known for making efficient use of the network~\cite{lim2020federated} (as raw data is never communicated) and has been applied at the scale of millions of devices~\cite{bonawitz2019towards}.


The performance (total training time) of Federated Learning is mainly dictated by two factors:
the time each training round takes (\textit{round duration}), and the number of rounds executed.
We will focus on the former, as the latter can be affected by factors such as the model and quality of available data, which are outside the scope of this work.
The time a device takes on training (its \textit{cost}) is related to its communication time with the server and its computation time (which depends on the device's characteristics, the machine learning model, and the amount of local data).
As each round requires a synchronization with the central FL server, the duration of a round is dictated by the slowest participating devices, also known as \textit{stragglers}~\cite{sprague2018asynchronous,mohammadi2019computation,wang2020optimize}.
The main issue with stragglers is that they can be much slower than other devices due to the heterogeneous nature of mobile devices.
For this reason, improving the performance of a training round requires minimizing the execution time on the slowest devices, which can be seen as a makespan minimization problem~\cite{leung2004handbook}.


Although communication times have been thought to dominate the execution time~\cite{mcmahan2017communication, lim2020federated}, recent research has shown that, in practice, computation is the bottleneck~\cite{wang2020optimize}.
In order to control the computation time of a device, one may control how much data (e.g., number of mini-batches) it uses for training based on a lower and upper limit.
This can create more data imbalance in the system, but this has not been an issue in previous works~\cite{mcmahan2017communication}, especially when data is independent and identically distributed~(i.i.d.)~\cite{wang2020optimize}.
Lower limits can enforce that all devices participate at a certain level in the training, while upper limits can be set by the devices (to due limited data or processing time available for training) or by the server (e.g., to avoid draining the battery of a device).

Deciding how much data each device should use for training in a round
is akin to assigning tasks to resources.
Given that scheduling problems are often NP-Complete or NP-Hard~\cite{leung2004handbook, casanova2008parallel}, one may expect that this problem also falls in these complexity classes and try to propose heuristics to find approximate solutions.
However, in this work we show that an optimal assignment of data to FL devices can be found in polynomial time.
We present this problem in a similar formulation to the problem of scheduling identical and independent tasks (1-D data) on heterogeneous resources~\cite{casanova2008parallel}, and we propose an algorithm --- named OLAR --- that computes an optimal assignment. 
In this context, our main contributions are:

\begin{itemize}

    \item[{$\star$}] We present a formulation of this optimization problem with lower and upper limits, and arbitrary, non-decreasing cost functions per resource.
    \item[{$\star$}] We propose a polynomial-time algorithm named OLAR, and we prove its optimality.
    \item[{$\star$}] We evaluate the quality of the assignment and the scheduling time of algorithms from the state of the art over varied scheduling scenarios (number of tasks, resources, kinds of resource cost functions, presence or absence of limits).
    \item[{$\star$}] We provide the algorithms' implementations and experimental data for use and reproduction of the results~\cite{gitrepo}.

\end{itemize}


The remaining of this paper is organized as follows:
Section~\ref{sec:rw} presents related work on FL and scheduling.
Section~\ref{sec:algo} describes the problem formulation, our algorithm and its proof of optimality.
Section~\ref{sec:exp} details our experimental evaluation, and Section~\ref{sec:conc} presents concluding remarks.

\section{Related Work}
\label{sec:rw}

There is a large volume of work on the improvement of Federated Learning mechanisms.
We point our readers to the survey by Lim et al.~\cite{lim2020federated} for a comprehensive view on the subject.
We will focus our attention here to works that can be compared to or that can benefit from OLAR.
We also discuss some points related to device profiling (to obtain costs) and the use of lower and upper limits of tasks (data) per device.

\textbf{Standard FL.}
FederatedAveraging (FedAvg) was proposed by Brendan McMahan et al.~\cite{mcmahan2017communication} when the concept of FL was introduced.
Given an equal distribution of data among devices, FedAvg was shown to reduce the number of rounds for convergence when compared to a baseline distributed stochastic gradient descent mechanism.
As it focuses on the number of rounds for convergence, 
FedAvg does not optimize the rounds' duration.
Given its importance, we consider FedAvg with an equal distribution of data as a baseline for comparison in our experiments (Section~\ref{sec:exp}). 
We also extend it to consider lower and upper limits of tasks per resource.

\textbf{Resource selection.}
Traditionally, the FL server chooses a subset of resources uniformly at random for a round~\cite{mcmahan2017communication}.
Other strategies for choosing resources, such as round robin, proportional to the signal-to-noise ratio~\cite{yang2019scheduling}, and age-based~\cite{yang2020age} have been proposed to try to accelerate convergence.
OLAR can work together with these approaches by assigning tasks to the chosen resources and minimizing round duration.

Nishio and Yonetani~\cite{nishio2019client} aim to limit round duration by selecting only a subset of devices that can participate.
Given a deadline to the round, they propose a greedy heuristic (FedCS) to maximize the number of participant devices.
Similarly, Shi et al.~\cite{shi2020device} propose a greedy heuristic to minimize round duration by choosing participating devices that do not pose a problem for communication.
Xia et al.~\cite{xia2020multi} employ Reinforcement Learning~(RL) to choose devices that help minimize the total training time.
In contrast, OLAR minimizes the round duration by selecting how much data each of the participating devices should use.

\textbf{Other objectives and decisions.}
Wang et al.~\cite{wang2020optimize, wang2020efficient} propose algorithms to minimize the training time and accuracy loss by controlling the assignment of tasks to resources (as we do). 
For the scenario with i.i.d.~data, they propose Fed-LBAP, an algorithm that uses binary search to find a minimal round duration that fits all data required.
We examine Fed-LBAP and an adaptation that considers lower and upper limits in Section~\ref{sec:exp}.
Likewise, Yang et al.~\cite{yang2019energy} use a similar binary search scheme to try to minimize the total training time, but based on parameters of bandwidth, transmission time, and target accuracy.

Zhan et al.~\cite{zhan2020experience} focus on minimizing a combination of total training time and energy consumption using RL to control the processor clock frequency on the devices.
Similarly, Anh et al.~\cite{anh2019efficient} use RL to decide how much data and energy each device should use in order to minimize the total training time and energy consumption on devices.
We plan to investigate if OLAR can be adapted to consider energy in the future.

\textbf{Deadline determination.}
Li et al.~\cite{li2019smartpc} propose a two-level mechanism (\textit{SmartPC}) that aims to balance the total training time and the energy consumption on the devices.
At a global level, it tries to determine a deadline for the round that should be respected by a part of the devices.
In this context, OLAR could be trivially adapted to maximize the number of tasks can be assigned to resources while respecting the chosen deadline.

\textbf{Distributed Learning.}
Mohammadi Amiri and Gunduz~\cite{mohammadi2019computation} present scheduling algorithms and bounds to solve issues with stragglers in the context of Distributed Learning.
Although Distributed Learning is similar to FL in some ways, it requires data to be shared among devices.
As sharing goes against the principles of FL, their algorithms cannot be applied in our context.
In any case, OLAR could be adapted to this scenario (e.g., by adding the cost of data transmission) and still provide optimal assignments.

\textbf{Asynchronous strategies.}
Some approaches avoid stragglers by removing the synchronization barrier used at each round. 
Damaskinos et al.~\cite{damaskinos2020fleet} introduced a new algorithm named \textit{Adaptive Stochastic Gradient Descent} (\textit{AdaSGD}) in order to do asynchronous learning.
Sprague et al.~\cite{sprague2018asynchronous} also propose their own asynchronous FL algorithm.
Ways to translate the benefits of OLAR to asynchronous FL are left as future work.

\textbf{Cost profiling.}
Collecting the cost information from devices (i.e., the time they take to communicate and train with a certain amount of data) could be seen as a barrier to implementing cost-aware scheduling algorithms.
Nonetheless, tools such as I-Prof~\cite{damaskinos2020fleet} could be utilized to solve this issue.
Additionally, Wang et al.~\cite{wang2020efficient} have shown that cost information can be profiled with a high degree of accuracy.

\textbf{Lower and upper limits.}
The effects of setting lower and upper limits to the number of tasks per resource have not been studied previously.
Yet, we see that limits can play an integral part in FL systems.
Upper limits come naturally from limited storage, processing time, or battery in mobile devices.
Limits could be use to enforce fairness constraints~\cite{xia2020multi}, as 
lower limits enforce device participation, while upper limits avoid an overrepresentation of data from some better-performing devices~\cite{lim2020federated}.
Besides, when setting incentives to attract devices to participate in training~\cite{kang2019incentive}, the use of limits can help define the scope of utilization of the devices for users.

\section{Optimal Scheduling Algorithm}
\label{sec:algo}

\stepcounter{mythmcounter}

The presentation of our novel scheduling algorithm for FL is organized as follows: we start by presenting some important definitions and explanations.  
We then present our scheduling algorithm and an analysis of its complexity. 
We finish by proving its optimality. 

\subsection{Definitions}
\label{subsec:defs}

Our scheduling problem can be formulated similarly to the problem of scheduling identical and independent tasks (1-D data) on heterogeneous resources~(\cite{casanova2008parallel}, Chapter 6.1).
Consider $\Tasks \in \Natural $a number of identical, independent, and atomic tasks (e.g., data units, mini-batches), and
a set~$\Resources$ of $\ResNum$ resources (e.g., mobile devices).
Each resource $i \in \Resources$ has its own lower and an upper limit on the number of tasks it can compute ($\Lower{i} \in \Natural$ and $\Upper{i} \in \Natural$, respectively),
and its own non-decreasing cost function $\Cost{i}{\cdot}:\Natural\rightarrow\Real$ that informs the cost of assigning a number of tasks to it.
Our problem is to find a task assignment $\Mapping{i} \in \Natural$ to each resource $i \in \Resources$
that minimizes the makespan $\Makespan$ (Eq.~\eqref{eq:makespan})
while assigning all tasks among the resources (Eq.~\eqref{eq:allmap})
and respecting the lower and upper limits (Eq.~\eqref{eq:all-respect}).
Throughout this text, we use the indexes $i$ and $k$ for resources and tasks, respectively.

\begin{equation}
    \Makespan \coloneqq \max\limits_{i \in \Resources} \Cost{i}{\Mapping{i}} 
    \label{eq:makespan}
\end{equation}

\begin{equation}
    \sum\limits_{i \in \Resources} \Mapping{i} = \Tasks
    \label{eq:allmap}
\end{equation}

\begin{equation}
    \Lower{i} \le \Mapping{i} \le \Upper{i},~\forall i \in \Resources
    \label{eq:all-respect}
\end{equation}

This scheduling problem is based on a few assumptions and ideas.
First, we consider that the cost functions are independent among resources, but they are all non-decreasing (Eq.~\eqref{eq:non-decreasing}).
For FL, the cost functions include the time taken to communicate the model between the FL server and the device, and the time to train the model with a certain amount of data.
Second, for there to be feasible solutions, the sums of the lower and upper limits have to surround $\Tasks$ (Eq.~\eqref{eq:respect-1}), and the lower limit for a resource cannot be greater to its upper limit ($\Lower{i} \le \Upper{i},~\forall i \in  \Resources$).
We could enforce that no resource would receive more tasks than its upper limit by setting the cost of any exceeding tasks as infinity (Eq.~\eqref{eq:inf}).
Third and final, there may be multiple optimal solutions for a given scenario.
Our focus is to find one of them, and not to list them all.

\begin{equation}
    \Cost{i}{k} \le \Cost{i}{k+1},~\forall i \in  \Resources, k \in \Natural
    \label{eq:non-decreasing}
\end{equation}

\begin{equation}
    \LowV \le \Tasks \le \UpV, ~~\LowV \coloneqq \sum\limits_{i \in \Resources} \Lower{i}, ~~ \UpV \coloneqq \sum\limits_{i \in \Resources} \Upper{i}
    \label{eq:respect-1}
\end{equation}

\begin{equation}
    \Cost{i}{k} = + \infty, ~~ \forall i \in \Resources, k > \Upper{i} 
    \label{eq:inf}
\end{equation}

In order to make it easier to explain and prove the optimality of our algorithm,
we extend our notation to consider the makespan (Eq.~\eqref{eq:makespan-t}) and the assignment of tasks to resources (Eq.~\eqref{eq:assignment-t}) at step $t \le \Tasks$.
At this step, $t$ tasks have been assigned to resources.
Accordingly, $\MakespanT{\Tasks} = \Makespan$.

\begin{equation}
    \MakespanT{t} \coloneqq \max\limits_{i \in \Resources} \Cost{i}{\MappingT{i}{t}} 
    \label{eq:makespan-t}
\end{equation}

\begin{equation}
    \sum\limits_{i \in \Resources} \MappingT{i}{t} = t
    \label{eq:assignment-t}
\end{equation}

\subsection{Algorithm}
\label{subsec:algo}

\begin{figure*}[!htb]
    \centering 
    \begin{subfigure}{0.31\textwidth}
      \includegraphics[width=0.8\textwidth]{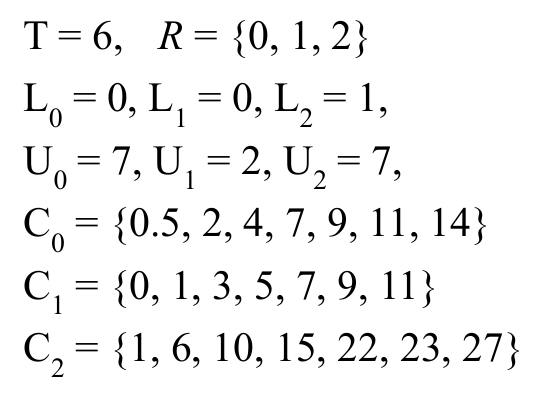}
      \caption{Parameters of the example.}
      \label{fig:0-1}
    \end{subfigure}\hfill 
    \begin{subfigure}{0.31\textwidth}
      \includegraphics[width=0.8\textwidth]{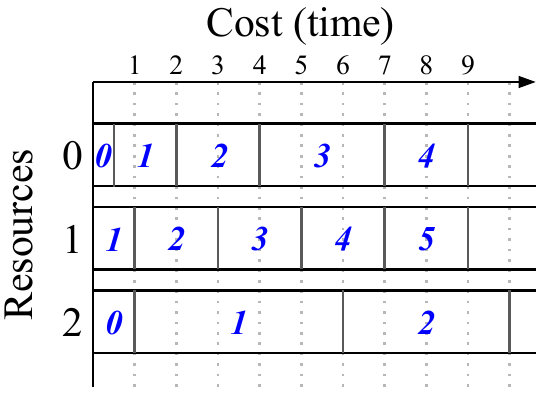}
        \caption{Gantt chart of the costs to map a number of tasks to each resource (\textit{\textcolor{blue}{italics, in blue}}).}
      \label{fig:0-2}
    \end{subfigure}\hfill 
    \begin{subfigure}{0.31\textwidth}
      \includegraphics[width=0.8\textwidth]{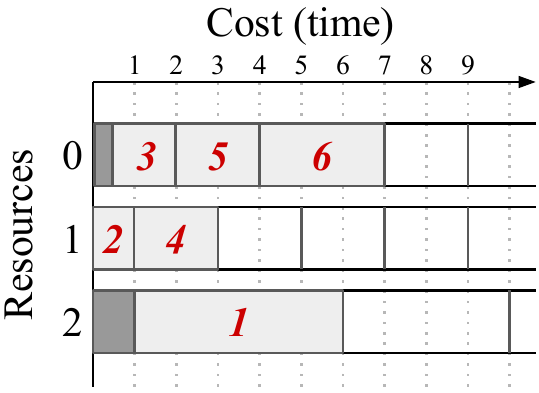}
        \caption{Gantt chart with the order the tasks are assigned with OLAR (\textit{\textcolor{red}{italics, in red}}).}
      \label{fig:0-3}
    \end{subfigure}
      \caption{Example of task assignment respecting lower and upper limits with OLAR.}
    \label{fig:0}
\end{figure*}

Our algorithm, named OLAR for \textit{OptimaL Assignment of tasks to Resources}, is 
similar to the solution of 
to the problem of scheduling identical and independent tasks (1-D data) on heterogeneous resources (\cite{casanova2008parallel}, Algorithm 6.2), but it includes different costs functions, and limits.
OLAR employs dynamic programming to compute an optimal final assignment by iteratively finding optimal assignments for an increasing number of tasks.
Its main idea is based on the notion of assigning the next task $t+1$ to a resource $j$ that would minimize the makespan $\MakespanT{t+1}$ (Eq.~\eqref{eq:find-j}).
OLAR is further detailed in Algorithm~\ref{alg:olar}.

\begin{equation}
    j = \argmin\limits_{i \in \Resources} \Cost{i}{\MappingT{i}{t}+1}
    \label{eq:find-j}
\end{equation}

\begin{algorithm}[!hb]
    \DontPrintSemicolon
    \KwData{Tasks $\Tasks$, Resources $\Resources$, Cost functions $\Cost{i}{\cdot}$,
    Lower and Upper limits $\Lower{i}$ and $\Upper{i}$ ($i \in \Resources$)}
    \KwResult{Assignment of tasks to resources $\Mapping{i}$ ($i \in \Resources$)}
    $h \leftarrow min$-$heap()$ \Comment*[r]{Heap sorted by cost}
    \For{$i \in \Resources$}{
        $\Mapping{i} \leftarrow \Lower{i}$ \Comment*[r]{Resources start at their lower limit}
        \Comment{Checks if the resource can receive more tasks}
        \If{$\Mapping{i} < \Upper{i}$}{
            \Comment{Inserts the cost of the next task on $i$}
            $h.push(\Cost{i}{\Mapping{i}+1}, i)$ 
        }
    }
    \For{$t$ from $\LowV+1$ to $\Tasks$}{
        \Comment{Extracts the next optimal assignment (Eq.~\eqref{eq:find-j})}
        $(c, j)  \leftarrow h.pop()$ \;
        $\Mapping{j} \leftarrow \Mapping{j} + 1$ \Comment*[r]{Assigns $t$ to $j$} 
        \Comment{Checks if the resource can receive more tasks}
        \If{$\Mapping{j} < \Upper{j}$}{
            \Comment{Inserts the cost of the next task on $j$}
            $h.push(\Cost{j}{\Mapping{j}+1}, j)$ 
        }
    }
\caption{OLAR}
\label{alg:olar}
\end{algorithm}

Algorithm~\ref{alg:olar} starts by setting up a minimum heap to store the costs of adding a task to each resource (line~1).
It then initializes the assignment of tasks to resources with their lower limits (line~3), and adds the cost of assigning their next task to the heap if the resource can receive more tasks (line~5).
Its main loop (lines~8--14) makes the optimal assignment of one task at a time
by getting one of the resources with the minimum cost to receive a new task (line~9),
assigning the additional task to it (line~10), and updating its cost on the heap (line~12) if it can still receive more tasks.
After all iterations of the main loop, all tasks will have been assigned to a resource and the algorithm finishes.

Fig.~\ref{fig:0} illustrates an example of a schedule computed by OLAR.
The parameters in Fig.~\ref{fig:0-1} indicate that each resource has its own cost function, and that resource~$2$ must receive at least one task, while resource~$1$ cannot receive more than two.
The Gantt chart in Fig.~\ref{fig:0-2} shows that resource~$2$ could be a straggler in this scenario, as its costs are much larger than the costs of resources~$0$ and~$1$. For instance, if we were to assign two tasks to each resource, the makespan would be dictated by~$\Cost{2}{2}=10$.
In contrast, OLAR would start by first assigning one task (\textit{\textcolor{red}{1}} in Fig.~\ref{fig:0-3}) to resource~$2$ to respect its lower limit (lines 2--7 in Algorithm~\ref{alg:olar}).
Then OLAR would follow its second loop and assign one task (\textit{\textcolor{red}{2}}) to resource~$1$, one task (\textit{\textcolor{red}{3}}) to resource~$0$, and another task (\textit{\textcolor{red}{4}}) to resource~$1$.
At this moment, OLAR would stop considering resource~$1$ for new assignments as it has reached its upper limit (line 11).
It follows assigning the remaining tasks (\textit{\textcolor{red}{5}} and \textit{\textcolor{red}{6}}) to resource~$0$, as it has the minimum cost.
This results in an assignment with $\Makespan=7$, which is optimal considering the restrictions.

\subsubsection{Complexity}

The time complexity of OLAR can be computed as follows:
the initialization loop (lines 2--7) runs for $\ResNum$ iterations;
each of its iterations includes an insertion in the heap, which takes $\Theta(1)$ for a Binomial heap; 
this makes the complexity of the first loop to be $\Theta(\ResNum)$.
Meanwhile, the main loop (lines 8--14) runs for $\Theta(\Tasks)$ iterations;
each iteration includes a removal from the heap ($\Theta(\log \ResNum)$), while the other operations are $\Theta(1)$.
Combining these two loops, we can conclude that OLAR is $\Theta(\ResNum + \Tasks \log \ResNum)$, or $\Theta(\Tasks \log \ResNum)$ for $\Tasks \ge \ResNum$.

We can notice that the general behavior of Algorithm~\ref{alg:olar} is similar to sorting the costs of assigning up to $\Tasks$ tasks to each resource,
and then iteratively traversing the resulting array to assign the $k^\text{th}$ task to the resource with the $k^\text{th}$-smallest cost.
Nevertheless, merging $\ResNum$ sorted arrays of size $\Tasks$ would require $O(\ResNum \Tasks\log \Tasks)$ operations.
OLAR avoids this complexity by keeping a minimum heap of $\ResNum$ elements only, and by iteratively removing the minimum and inserting a new item to the heap $\Tasks$ times only. 
Moreover, given that finding an optimal solution requires the minimum of $\ResNum$ values $\Tasks$ times, we present Conjecture~\ref{conj:1}.
This conjecture could be disproved by finding an optimal algorithm that does not require minimum values or sorting\footnote{We require comparison-based sorting algorithms because costs are non-negative real numbers.}.
\begin{conjecture}
    The lower bound of any optimal sequential algorithm for this scheduling problem is $\Omega(\ResNum\log \ResNum)$ for $\ResNum = \Tasks$.
    \label{conj:1}
\end{conjecture}


The space complexity of OLAR is dominated by the presence of the $\ResNum$ cost functions of size $\Tasks+1$, as all other data structures are $\Theta(\ResNum)$.
This results in a space complexity of $\Theta(\ResNum \Tasks)$.
Nevertheless, if we know the functions that generate the costs for each resource, we could compute $\Cost{i}{k}$ during execution.
If each cost could be computed in $\Theta(1)$, then the space complexity of OLAR would be reduced to $\Theta(\ResNum)$ with no change to its asymptotic time complexity.

\subsection{Proof of optimality}
\label{subsec:proof}

We prove that OLAR is optimal by first proving that the base cases are optimal in Lemma~\ref{lemma:0} and Corollary~\ref{lemma:zetta}, and then by proving that each step of the algorithm is optimal in Lemma~\ref{lemma:t}.
These proofs are combined in a proof by induction in Theorem~\ref{theorem}.

\begin{lemma}\label{lemma:0}
    $\MakespanT{0}$ is optimal.
\end{lemma}
\begin{proof}
    The proof is trivial. As there is no assignment decision done before or at this step, the only (therefore, optimal) solution is assigning zero tasks to all resources.
\end{proof}
\begin{corollary}\label{lemma:zetta}
    $\MakespanT{\LowV}$ is optimal.
\end{corollary}
\begin{proof}
    Similarly to Lemma~\ref{lemma:0}, there is only one assignment of $\LowV$ tasks to resources that respects the lower limits of tasks for all resources (i.e., $\MappingT{i}{\LowV} = \Lower{i}, \forall i \in \Resources$), therefore it is optimal.
\end{proof}

\begin{lemma}\label{lemma:t}
    If $\MakespanT{t}$ is optimal, $\MakespanT{t+1}$ is optimal.
\end{lemma}
\begin{proof}
\textit{By contradiction.} 
    Assume there is a resource $r \in \Resources\setminus\{j\}$ ($j$ is defined in Eq.~\eqref{eq:find-j}) such that mapping task $t+1$ to $r$ leads to a smaller makespan than mapping it to $j$. This would lead to the following inequality for their makespans:
\begin{equation}
\begin{aligned}
    \max \left( \max\limits_{i \in \Resources \setminus \{r\}} \Cost{i}{\MappingT{i}{t}}, \Cost{r}{\MappingT{r}{t}+1} \right) &< \\ 
    \max \left( \max\limits_{i \in \Resources \setminus \{j\}} \Cost{i}{\MappingT{i}{t}}, \Cost{j}{\MappingT{j}{t}+1} \right) &
    \label{eq:big-diff}
\end{aligned}
\end{equation}

    Using the non-decreasing property of $\Cost{i}{\cdot}$ in Eq.~\eqref{eq:non-decreasing}, the associative property of the maximum operator, and the definition of $\MakespanT{t}$ in Eq.~\eqref{eq:makespan-t}, we can rewrite Eq.~\eqref{eq:big-diff} as

\begin{equation}
\begin{aligned}
    \max \left( \max\limits_{i \in \Resources \setminus \{r\}} \Cost{i}{\MappingT{i}{t}}, \max (\Cost{r}{\MappingT{r}{t}}, \Cost{r}{\MappingT{r}{t}+1}) \right) &< \\ 
    \max \left( \max\limits_{i \in \Resources \setminus \{j\}} \Cost{i}{\MappingT{i}{t}}, \max (\Cost{j}{\MappingT{j}{t}},\Cost{j}{\MappingT{j}{t}+1}) \right) & 
\end{aligned}
\end{equation}

\begin{equation}
\begin{aligned}
    \max \left( \max\limits_{i \in \Resources} \Cost{i}{\MappingT{i}{t}}, \Cost{r}{\MappingT{r}{t}+1} \right) &< \\ 
    \max \left( \max\limits_{i \in \Resources} \Cost{i}{\MappingT{i}{t}}, \Cost{j}{\MappingT{j}{t}+1} \right) & 
\end{aligned}
\end{equation}

\begin{equation}
    \max \left( \MakespanT{t}, \Cost{r}{\MappingT{r}{t}+1} \right) < 
    \max \left( \MakespanT{t}, \Cost{j}{\MappingT{j}{t}+1} \right) 
    \label{eq:the-diff}
\end{equation}

    We can split the analysis of Eq.~\eqref{eq:the-diff} in two parts:

    \textit{Part 1.}~$\MakespanT{t} \ge \Cost{j}{\MappingT{j}{t}+1}.$ This would mean
\begin{equation*}
    \MakespanT{t} \le \Cost{r}{\MappingT{r}{t}+1} < \MakespanT{t}
\end{equation*}
or
\begin{equation*}
    \Cost{r}{\MappingT{r}{t}+1} \le \MakespanT{t} < \MakespanT{t}
\end{equation*}

Both require $\MakespanT{t} < \MakespanT{t}$, which is a contradiction.

    \textit{Part 2.}~$\MakespanT{t} < \Cost{j}{\MappingT{j}{t}+1}.$ Using the definition of $j$ in Eq.~\eqref{eq:find-j}, this would mean
\begin{equation*}
    \MakespanT{t} \le \Cost{r}{\MappingT{r}{t}+1} < \min\limits_{i \in \Resources} \Cost{i}{\MappingT{i}{t}+1}
\end{equation*}
or
\begin{equation*}
    \Cost{r}{\MappingT{r}{t}+1} \le \MakespanT{t} < \min\limits_{i \in \Resources} \Cost{i}{\MappingT{i}{t}+1}
\end{equation*}

Both require $\Cost{r}{\MappingT{r}{t}+1} < \min\limits_{i \in \Resources} \Cost{i}{\MappingT{i}{t}+1}$ and, as $r \in \Resources$, this is a contradiction.

As all parts are contradictions, this means that Eq.~\eqref{eq:big-diff} is false, so $\MakespanT{t+1}$ is optimal.
\end{proof}

\begin{theorem}\label{theorem}
The makespan $\Makespan$ computed by OLAR is optimal.
\end{theorem}
\begin{proof}
    \textit{By induction.} Lemma~\ref{lemma:0} and Corollary~\ref{lemma:zetta} prove the optimality for the base case, while Lemma~\ref{lemma:t} proves it for the inductive step, so $\Makespan$ is optimal.
\end{proof}

\section{Experimental evaluation}
\label{sec:exp}

With the interest of comparing how OLAR and other scheduling algorithms from the state of the art perform for Federated Learning,
we have organized an experimental evaluation based on simulation.
All the data and code necessary to run the experiments and their analysis was implemented using Python~3 and made available online for the community~\cite{gitrepo}.
We chose Python due to its ease of implementation (with libraries such as \texttt{numpy}) and execution (which facilitates the reproduction of our experiments).
Whenever possible, we used \texttt{numpy} array operations to optimize the schedulers.
OLAR also uses \texttt{heapq} for its minimum heap.
As this module implements binary heaps, we adapted OLAR to use a \textit{heapify} operation instead of inserting costs iteratively (line~5 in Algorithm~\ref{alg:olar}) to keep its first loop within $O(n)$ operations.

Our experiments include four other scheduling algorithms, four kinds of cost functions for the resources, and a combination of scheduling scenarios with varying numbers of resources, tasks, and others.
We compare the algorithms based on the quality of their schedules (achieved makespans) and their execution times for what represents scheduling a single round of training.
We evaluate these metrics for scenarios without and with the inclusion of lower and upper limits of tasks per resource.
We explain each of these points in the next sections.
Additional information related to the execution platform and special parameters are detailed in Appendix~\ref{appendix:rep}.

\subsection{Scheduling Algorithms}

Besides OLAR, our experiments include four limit-unaware algorithms:

\begin{enumerate}
    \item \textbf{FedAvg}~\cite{mcmahan2017communication} distributes the tasks equally among the resources in a cost-oblivious manner.
    It serves as our baseline for comparisons.

    \item \textbf{Fed-LBAP}~\cite{wang2020optimize} was created to schedule i.i.d. data following ideas related to the \textit{linear bottleneck assignment
problem} (hence, LBAP).
    The algorithm sorts the costs for all resources and numbers of tasks, and then performs a binary search to find the minimum cost that includes the assignment of all tasks.
    Each step of the binary search verifies if enough tasks could be assigned respecting the current maximum cost.
 
    \item \textbf{Proportional($k$)} considers the cost of mapping $k$ tasks to each resource. 
    These costs are used to compute an inverse proportion of the number of tasks to be received (the higher the cost, the less tasks a resource receives).
    If any tasks are missing due to rounding errors, they are assigned
    one by one to the resources in order.

    \item \textbf{Random(seed)} generates a random number uniformly in the interval $[1,10)$
    for each resource based on a random number generator seed. 
    It uses the sum of all random numbers to find the proportion of
    tasks each resource will receive.
    If any tasks are missing due to rounding errors, it randomly adds
    them to the resources.
\end{enumerate}

Given the need to set the values of $\Mapping{i}$ ($i \in \Resources$), FedAvg, Proportional, and Random are $O(\ResNum)$ sequential algorithms.
Meanwhile, Fed-LBAP does $O(\ResNum\Tasks\log \ResNum\Tasks)$ operations~\cite{wang2020optimize}.

\subsection{Kinds of Resources (Cost Functions)}

We simulate resources (devices) whose costs follow four possible functions.
All include one or more parameters ($\alpha$, $\beta$, $\gamma$) that are randomly chosen from a 
uniform distribution in the interval $[1,10)$.

\begin{enumerate}
    \item \textit{Recursive} follows a function $f(x) = f(x-1) + \alpha_x$, where
    $\alpha_x$ is randomly chosen for each value of $x$.
    On average, the cost of assigning an additional task to this kind of resource is~$5.5$.
    This makes it the most constant of the tested cost functions.

    \item \textit{Linear} follows a function $f(x) = \alpha + \beta x$.

    \item \textit{Nlogn} follows a function $f(x) = \alpha + \beta x \log x$.

    \item \textit{Quadratic} follows a function $f(x) = \alpha + \beta x + \gamma x^2$.

\end{enumerate}

\subsection{Summary of the Experimental Scenarios}

We ran experiments to measure the achieved makespans and algorithm execution times with and without limits of tasks per resources.
Each experiment follows a slightly different organization: 

\begin{itemize}
\item \textit{Scenario 1} focuses on the makespan achieved when no limits are considered; it serves to compare OLAR to the other scheduling algorithms in their original environment.
\item \textit{Scenario 2} measures the execution time of the algorithms without the use of limits; it is intended to compare the times of the best implementation of each algorithm.
\item \textit{Scenario 3} captures the makespans when lower and upper limits of tasks per resources are present; it shows how the algorithms behave after being extended to handle limits.
\item \textit{Scenario 4} profiles the execution times of the extended algorithms; it emphasizes how lower and upper limits can affect the algorithm's performance.
\end{itemize}

We present the details of each scenario and their respective results in the next sections.

\subsection{Scenario 1: (achieved makespan, no limits)} 
\label{subsec:s1}

In this scenario, the algorithms had to schedule from $1,000$ to $10,000$ tasks (in increments of $100$) over $10$ and $100$ resources. 
The heterogeneous resources are organized in five groups: all recursive costs, all linear costs, etc, and one mixed group composed of the same proportion from the four cost functions\footnote{When $\ResNum=10$, we have $3$ Recursive, $3$ Linear, $2$ Nlogn, and $2$ Quadratic resources.}.
These experiments include results for Random with three different starting random seeds, and results for Proportional with $k = \{1, \floor*{\frac{\Tasks}{\ResNum}}, \Tasks\}$.

Fig.~\ref{fig:s1} summarizes the makespan results for Scenario~1.
Each figure presents the results for one group and number of resources.
The horizontal axis represents the number of tasks to assign, and the vertical axis represents
the makespan achieved by the assignment.
Each scheduler is represented by a line connecting the makespan achieved for consecutive numbers of tasks.
Each figure has its own scale for the vertical axis due to the particular costs of its group of resources.

\begin{figure}[!htb]
    \centering 
    \begin{subfigure}{0.48\columnwidth}
      \includegraphics[width=\columnwidth]{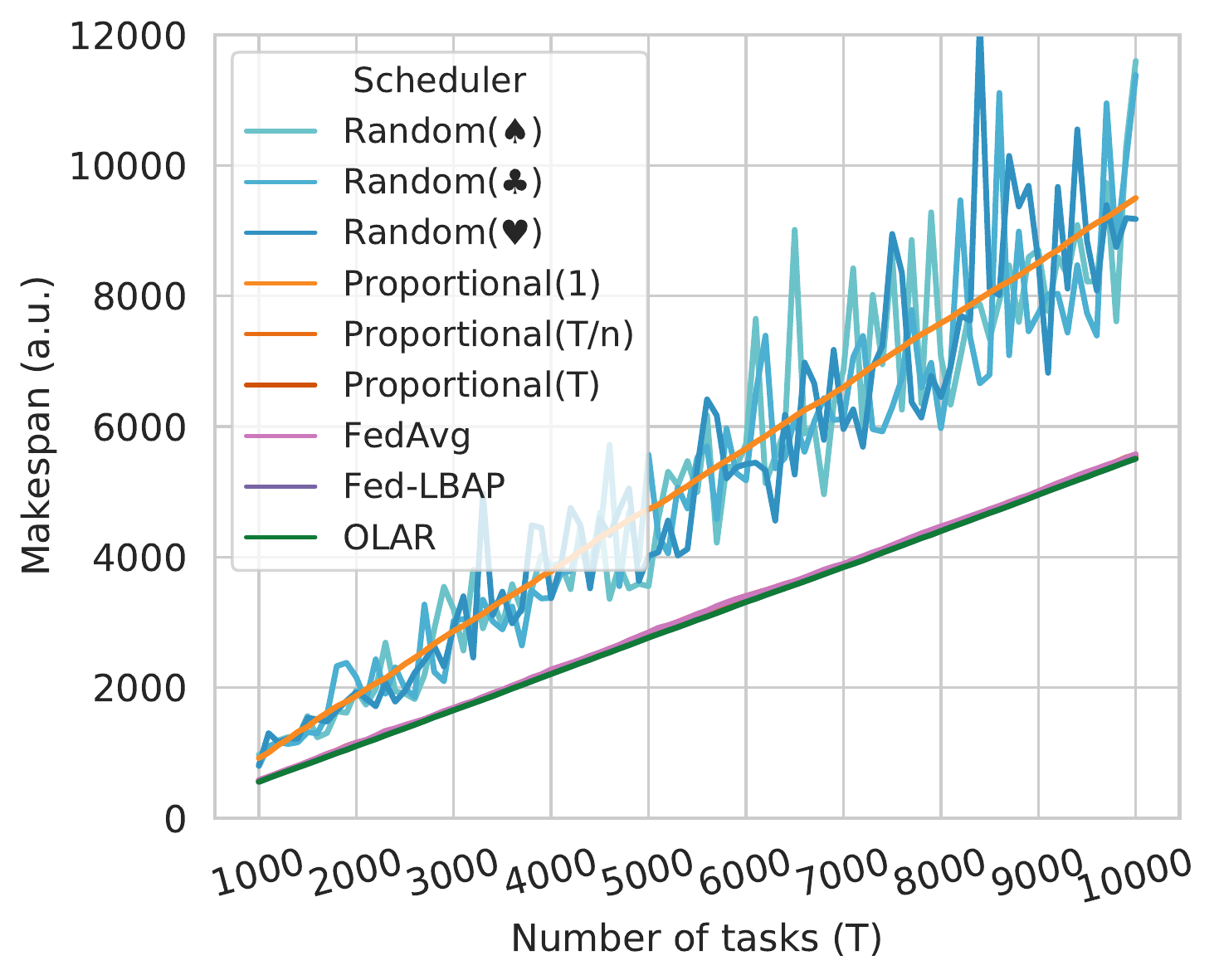}
      \caption{Recursive costs, $\ResNum=10$.}
      \label{fig:s1-recursive-10}
    \end{subfigure}\hfill 
    \begin{subfigure}{0.48\columnwidth}
      \includegraphics[width=\columnwidth]{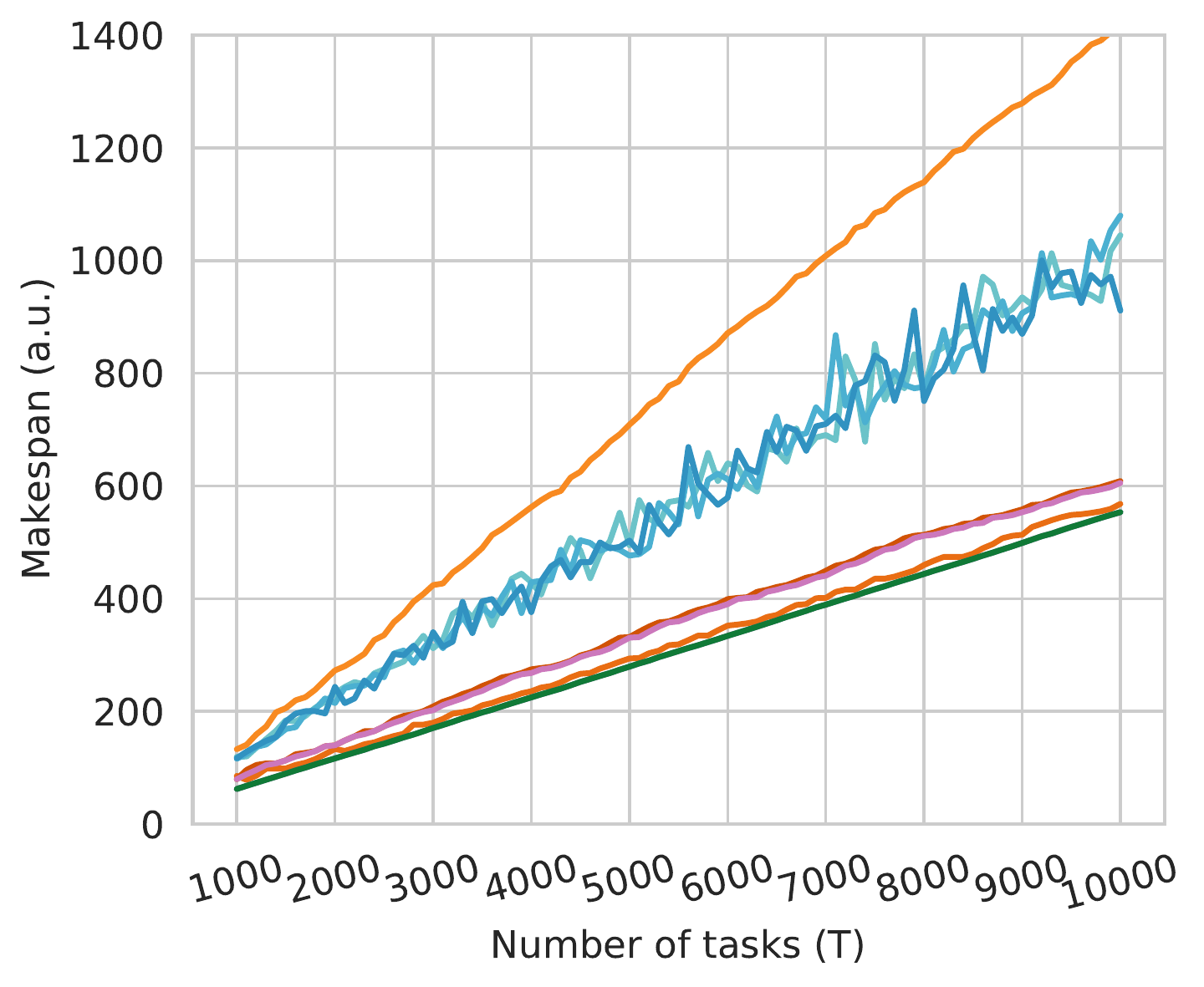}
      \caption{Recursive costs, $\ResNum=100$.}
      \label{fig:s1-recursive-100}
    \end{subfigure}
    \medskip
    \begin{subfigure}{0.48\columnwidth}
      \includegraphics[width=\columnwidth]{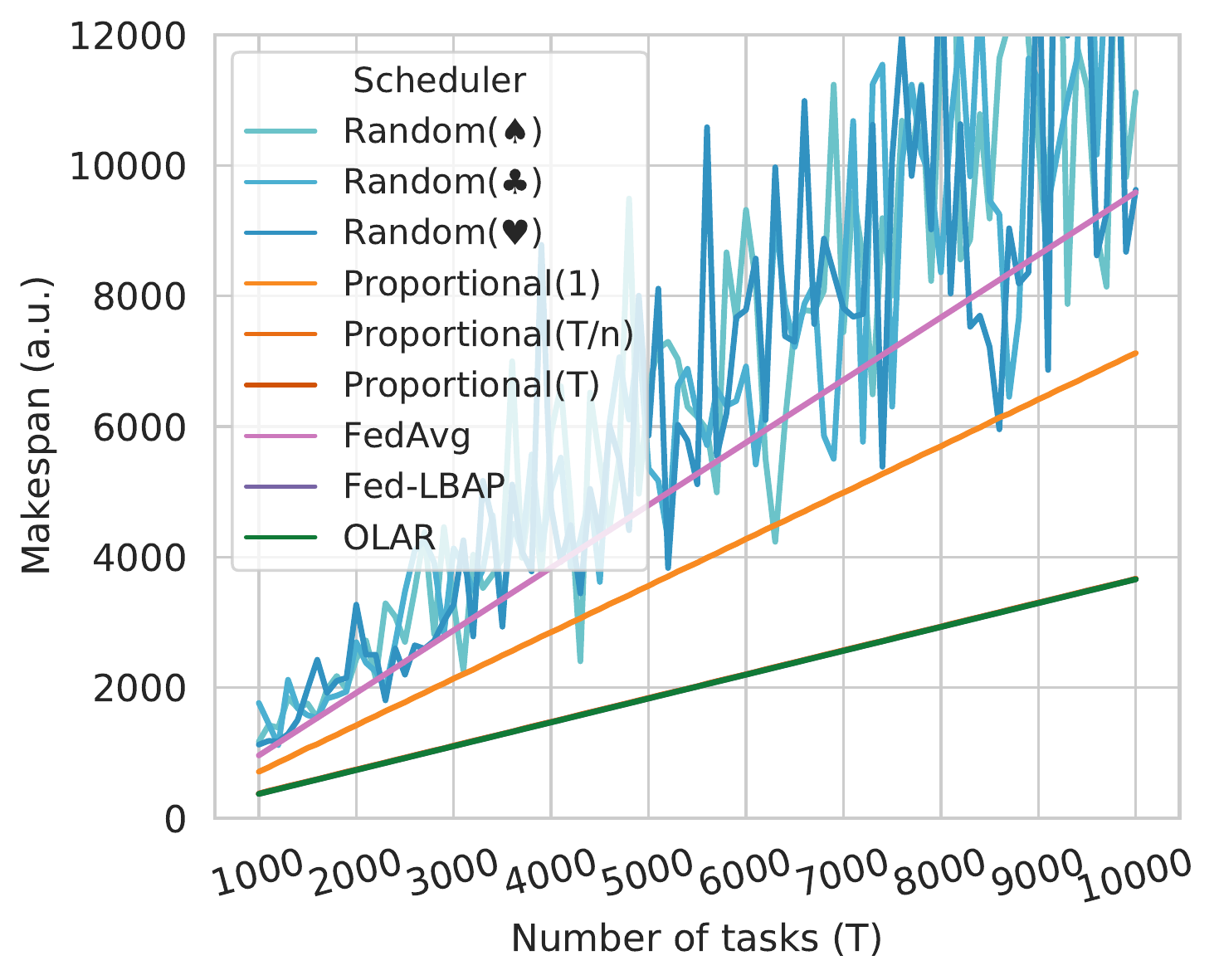}
      \caption{Linear costs, $\ResNum=10$.}
      \label{fig:s1-linear-10}
    \end{subfigure}\hfill 
    \begin{subfigure}{0.48\columnwidth}
      \includegraphics[width=\columnwidth]{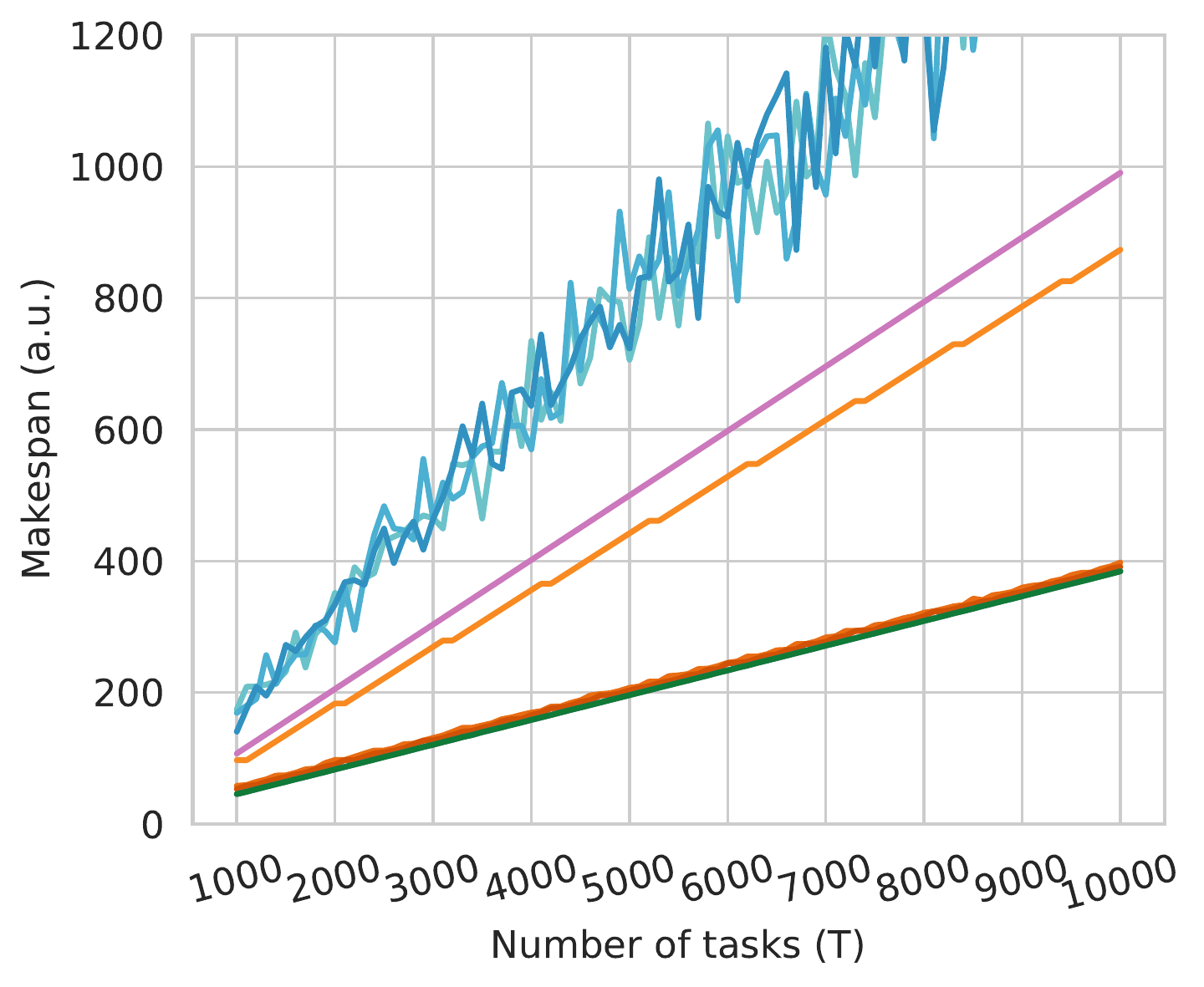}
      \caption{Linear costs, $\ResNum=100$.}
      \label{fig:s1-linear-100}
    \end{subfigure}
    \medskip
    \begin{subfigure}{0.48\columnwidth}
      \includegraphics[width=\columnwidth]{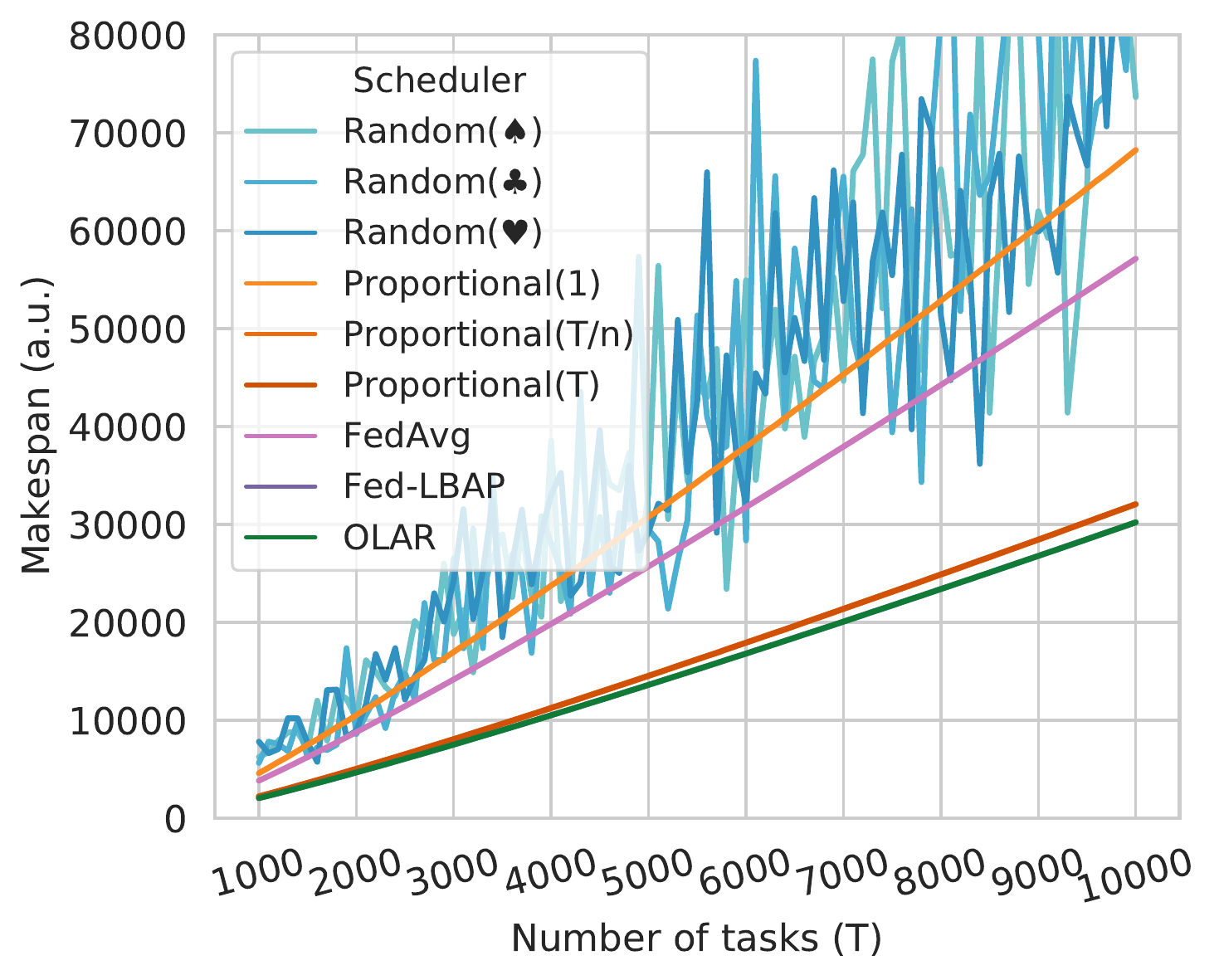}
      \caption{Nlogn costs, $\ResNum=10$.}
      \label{fig:s1-nlogn-10}
    \end{subfigure}\hfill 
    \begin{subfigure}{0.48\columnwidth}
      \includegraphics[width=\columnwidth]{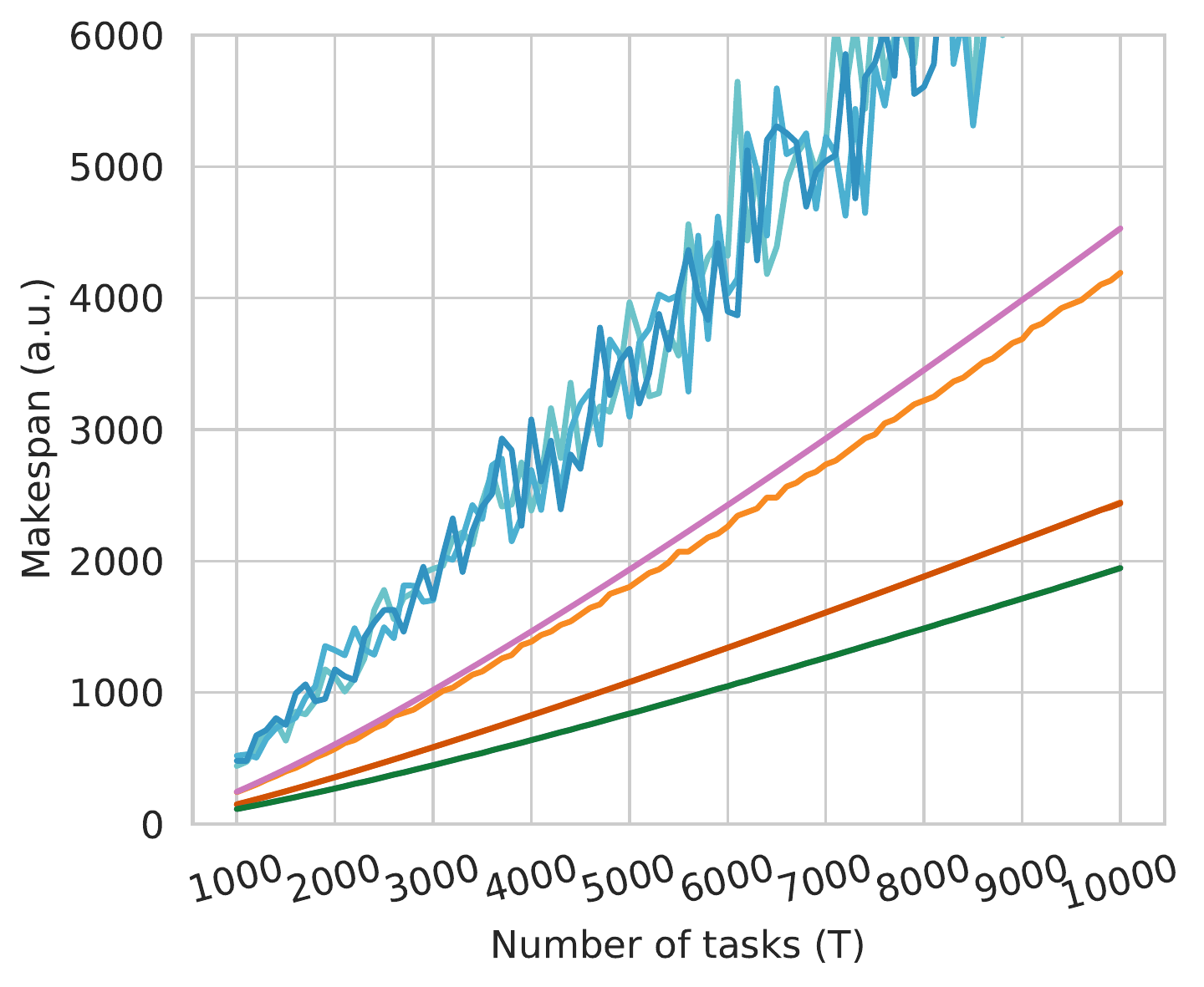}
      \caption{Nlogn costs, $\ResNum=100$.}
      \label{fig:s1-nlogn-100}
    \end{subfigure}
    \medskip
    \begin{subfigure}{0.48\columnwidth}
      \includegraphics[width=\columnwidth]{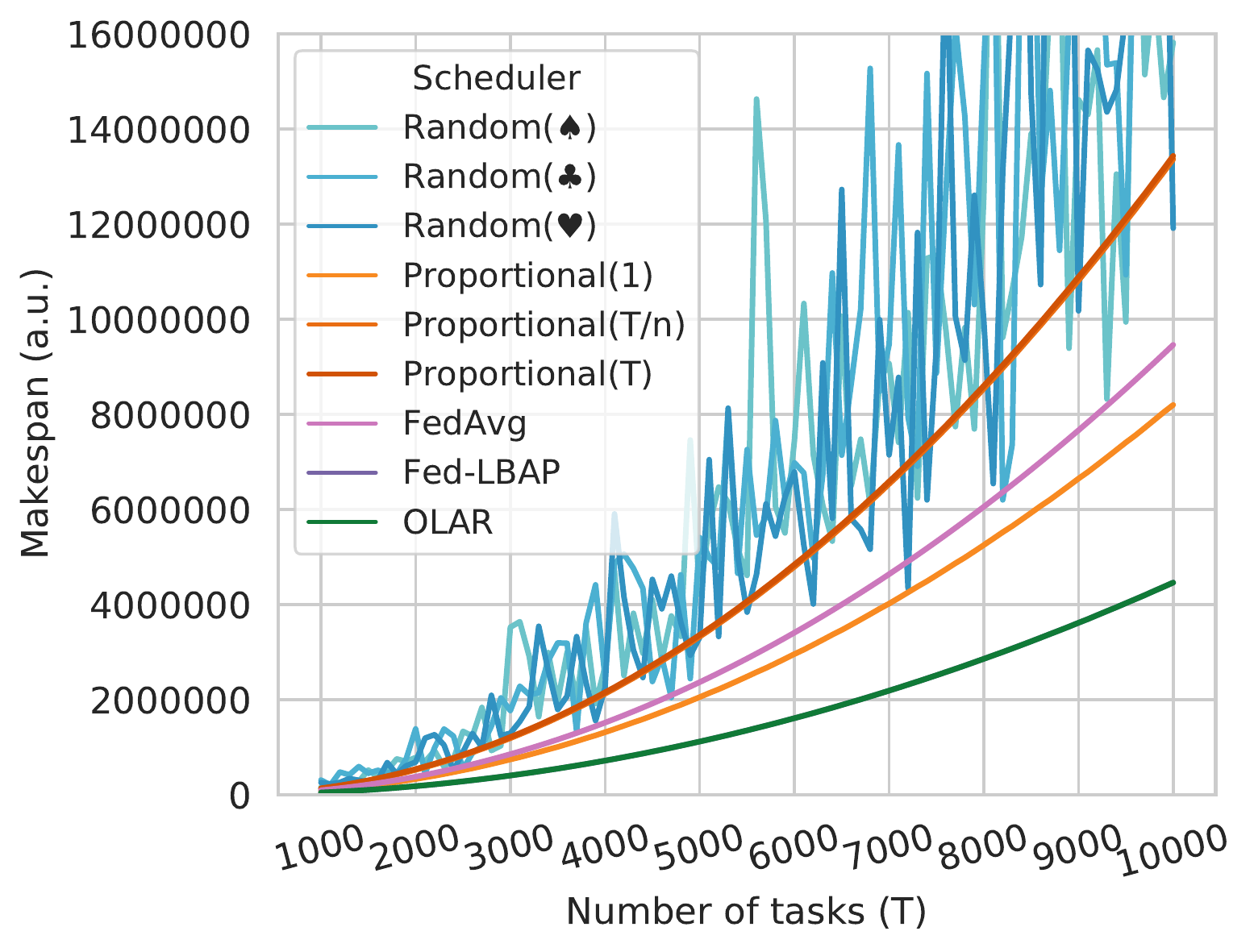}
      \caption{Quadratic costs, $\ResNum=10$.}
      \label{fig:s1-quadratic-10}
    \end{subfigure}\hfill 
    \begin{subfigure}{0.48\columnwidth}
      \includegraphics[width=\columnwidth]{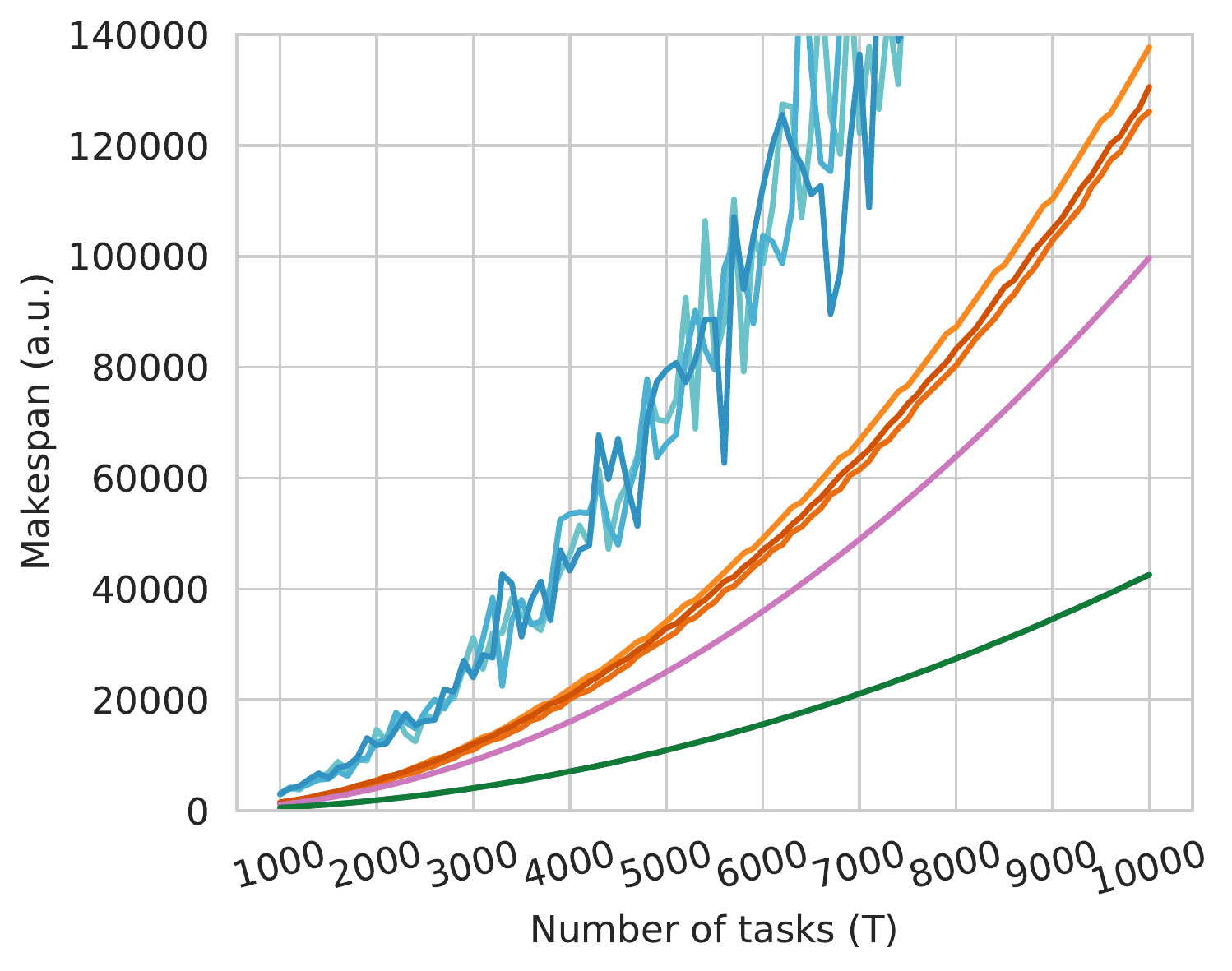}
      \caption{Quadratic costs, $\ResNum=100$.}
      \label{fig:s1-quadratic-100}
    \end{subfigure}
    \medskip
    \begin{subfigure}{0.48\columnwidth}
      \includegraphics[width=\columnwidth]{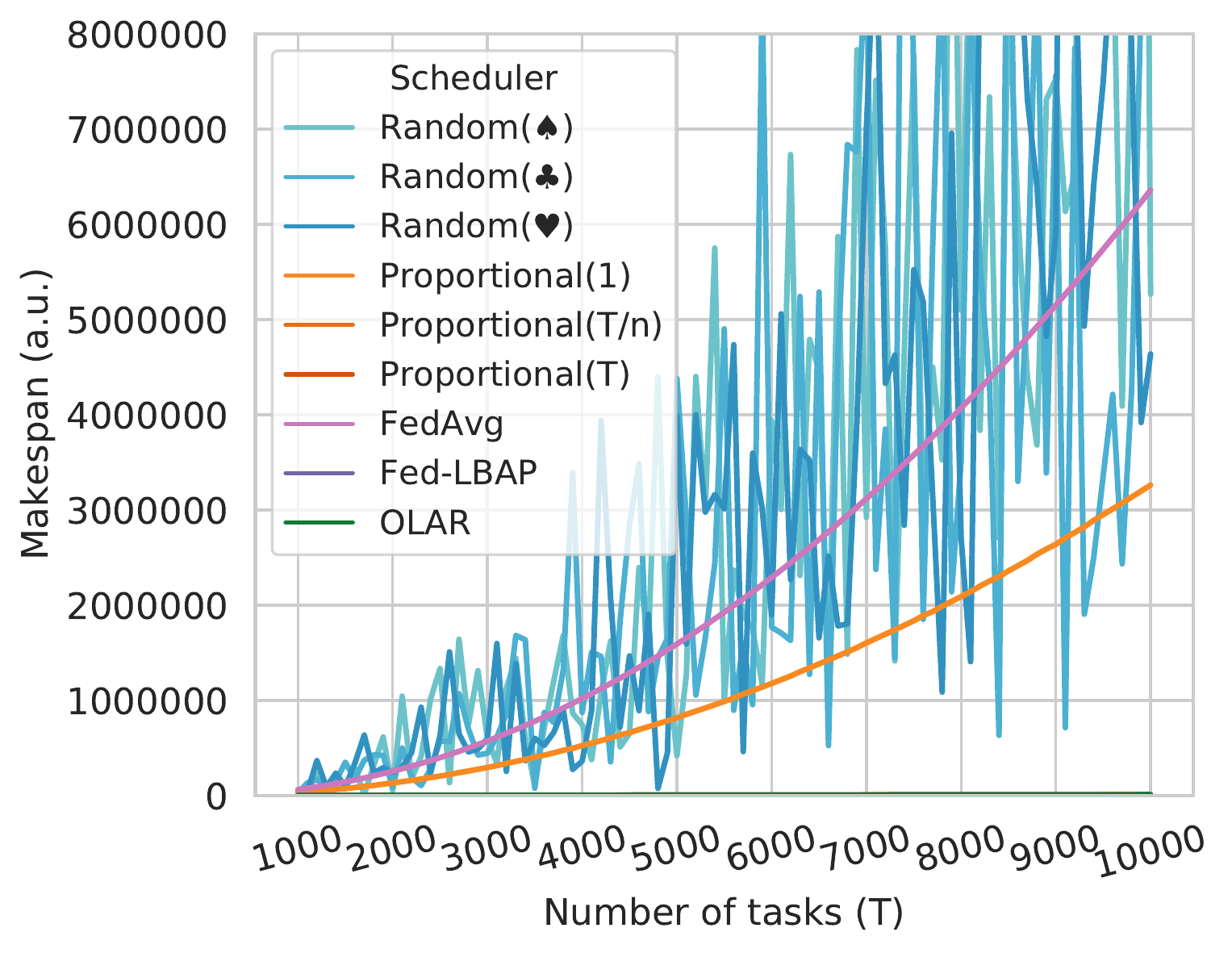}
      \caption{Mixed costs, $\ResNum=10$.}
      \label{fig:s1-mixed-10}
    \end{subfigure}\hfill 
    \begin{subfigure}{0.48\columnwidth}
      \includegraphics[width=\columnwidth]{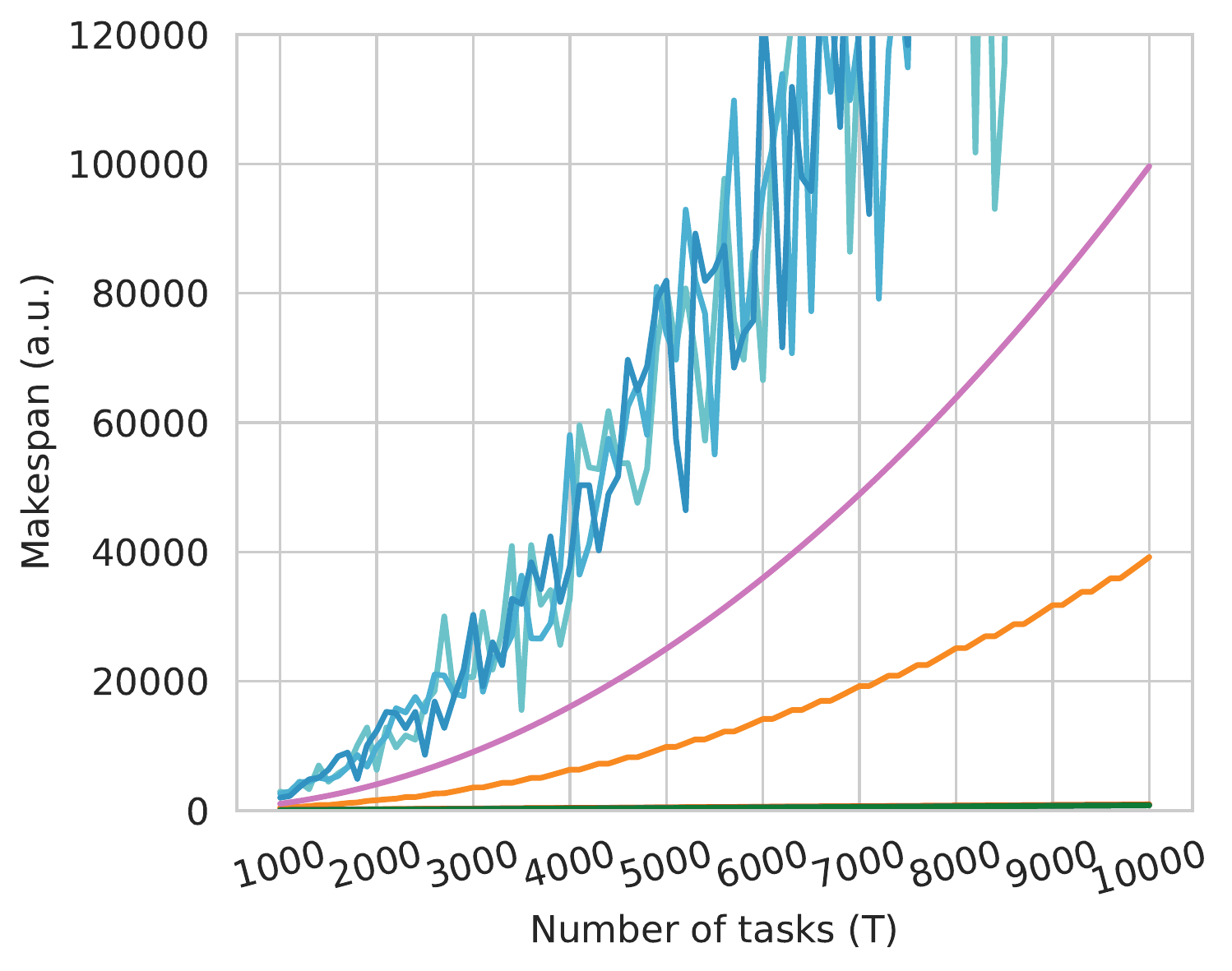}
      \caption{Mixed costs, $\ResNum=100$.}
      \label{fig:s1-mixed-100}
    \end{subfigure}
    \caption{Makespan results for Scenario~1.}
    \label{fig:s1}
\end{figure}

In this scenario, both OLAR and Fed-LBAP found optimal assignments for all cases, 
while none of the other algorithms found a single optimal assignment.
Regarding the other schedulers, we can notice that Random usually performs much worse than FedAvg. Additionally, their difference tends to increase as the number of tasks to be assigned grows, making it just not worth exploring Random assignments for this problem.
FedAvg performs close to optimal only for the resources with recursive costs (Figs.~\ref{fig:s1-recursive-10} and~\ref{fig:s1-recursive-100}), as this group is the most homogeneous among all the tested groups. In the other cases, assigning $\frac{\Tasks}{\ResNum}$~tasks to the most costly resource easily dominates the makespan.

Proportional performs close to optimal when resources with linear costs are present (Recursive, Linear, and Mixed groups) and $k$ is large.
This happens because it estimates linear costs for all resources, so large values of $k$ approximate well the real costs in these cases.
In general, the results with $k=\floor*{\frac{\Tasks}{\ResNum}}$ and $k=\Tasks$ are very similar (sometimes overlapping in the figures), while results with $k=1$ are much worse.
Nonetheless, when resources follow other behaviors, Proportional starts to get farther from the optimal solution.
This can be noticed for the groups with Nlogn costs (Figs.~\ref{fig:s1-nlogn-10} and~\ref{fig:s1-nlogn-100}), and especially for the resources with Quadratic costs, where it gets to perform worse than FedAvg (Fig.~\ref{fig:s1-quadratic-100}).

A final point to be taken from these results is that cost-aware scheduling algorithms are exceptionally important when resources follow different cost behaviors.
This is clearly illustrated in the results for resources with Mixed costs and $10$ resources (Fig.~\ref{fig:s1-mixed-10}), where FedAvg's makespan is about $800$ times worse than the optimal for $10,000$ tasks. 

\subsection{Scenario 2: (scheduling time, no limits)}
\label{subsec:s2}

The experiments in this scenario are split into two:
we first fix the number of resources at $100$, and vary the tasks from $1,000$ to $10,000$ in increments of $1,000$.
Then, we fix the number of tasks at $10,000$, and vary the number of resources from $100$ to $1,000$ in increments of $100$.
All resources follow Linear cost functions, as we assume that their costs should not have a major impact on the performance of the schedulers.
For each triple (scheduler, tasks, resources), we gather $50$ samples. Each sample is composed of $100$ runs of a scheduler measured using Python's \texttt{timeit} module.
The order that the samples are collected is randomized to reduce issues with interference and system jitter.

The average execution times for each triple are presented in Fig.~\ref{fig:s2}.
The vertical axis represents the execution time for each scheduler (ms, in log scale), while the horizontal axis represents the number of tasks and the number of resources in
Figs.~\ref{fig:s2-fixed-resources} and~\ref{fig:s2-fixed-tasks}, respectively.
Each scheduler is represented by a line connecting their execution times achieved for consecutive cases.
The average times for each scheduler for the smallest and largest cases are also presented in Tables~\ref{tab1} and~\ref{tab2}.


As can be noticed in these results, Random, Proportional, and FedAvg perform at least one, two, or three orders of magnitude faster than Fed-LBAP and OLAR.
This is to be expected, as the three algorithms perform $O(\ResNum)$ operations
(which is also the reason for no performance changes when increasing the number of tasks).
Moreover, these differences emphasize that generating random numbers is slower than only computing proportions, and computing proportions is slower than only doing a division.
The bumps seen for FedAvg in Fig.~\ref{fig:s2-fixed-tasks} come from the division $\frac{\Tasks}{\ResNum}$ leaving leftover tasks that have to be assigned in a second step.

\begin{figure}[!htb]
    \centering 
    \begin{subfigure}{0.48\columnwidth}
      \includegraphics[width=\columnwidth]{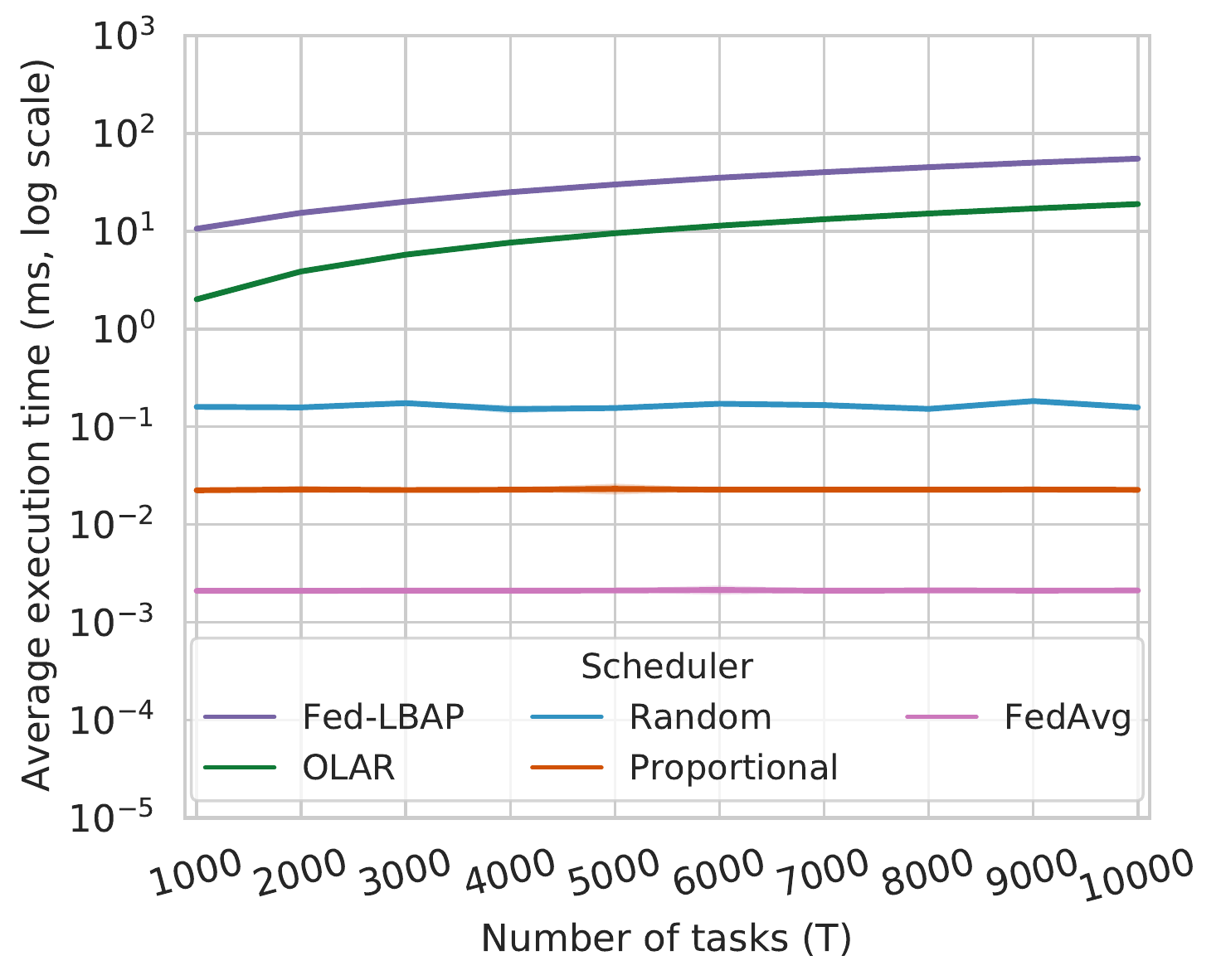}
      \caption{Fixed $\ResNum=100$.}
      \label{fig:s2-fixed-resources}
    \end{subfigure}\hfill 
    \begin{subfigure}{0.48\columnwidth}
      \includegraphics[width=\columnwidth]{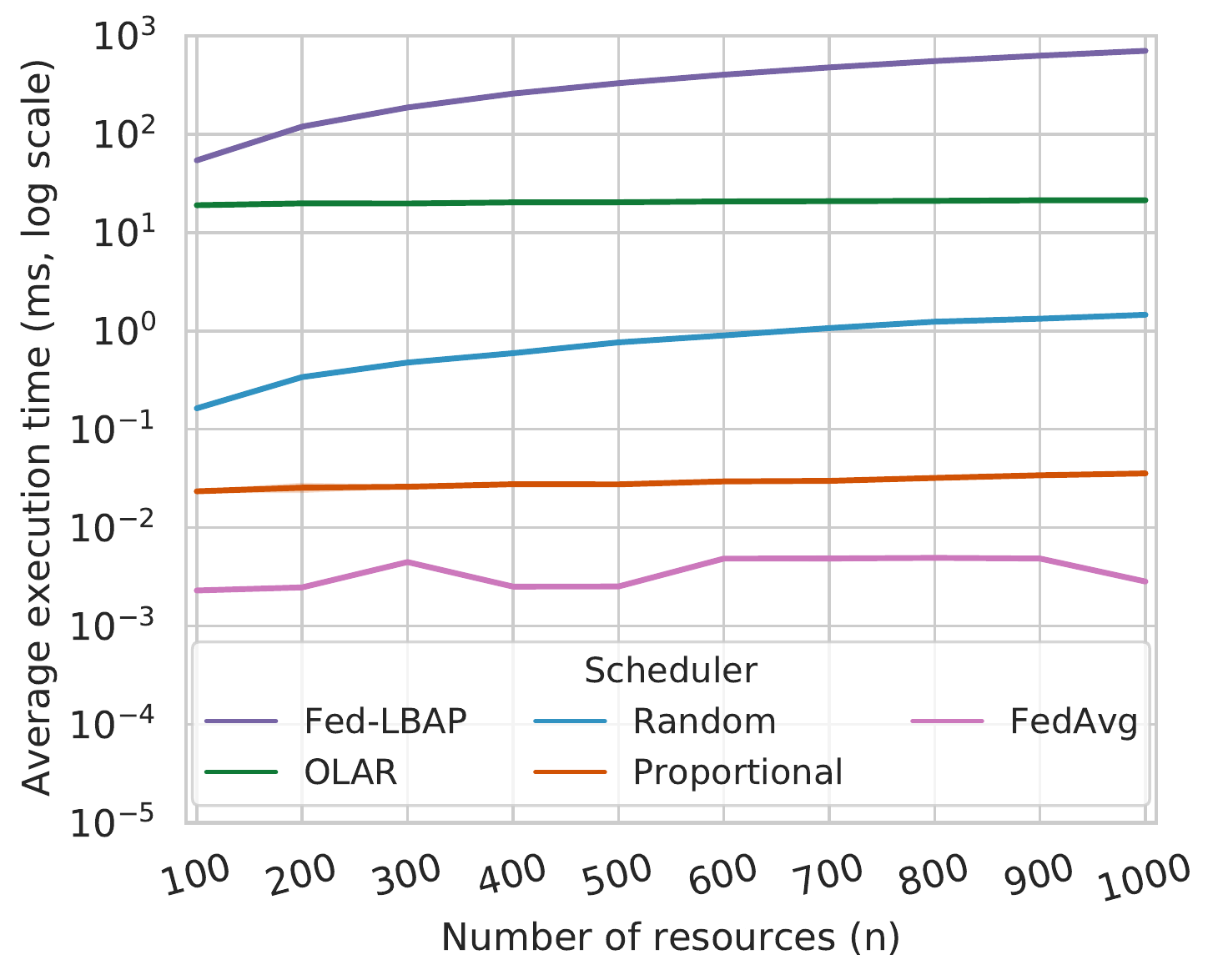}
      \caption{Fixed $\Tasks=10,000$.}
      \label{fig:s2-fixed-tasks}
    \end{subfigure}
     \caption{Scheduling algorithm times for Scenario~2.}
    \label{fig:s2}
\end{figure}

\begin{table}[!ht]
\caption{Average execution times (ms) with $\ResNum=100$.}
\resizebox{\columnwidth}{!}{
\begin{tabular}{r|rrrrr}
\multicolumn{1}{l}{\textit{\textbf{}}} & \multicolumn{1}{c}{\textbf{Fed-LBAP}} & \multicolumn{1}{c}{\textbf{OLAR}} & \multicolumn{1}{c}{\textbf{Random}} & \multicolumn{1}{c}{\textbf{Proportional}} & \multicolumn{1}{c}{\textbf{FedAvg}} \\
\midrule
\textit{$1,000$ tasks}                    & $10.592$                                & $2.009$                             & $0.160$                               & $0.022$                                     & $0.002$                               \\
\textit{$10,000$ tasks}                   & $55.091$                                & $18.930$                            & $0.158$                               & $0.023$                                     & $0.002$
\end{tabular}
}
\label{tab1}
\end{table}

\begin{table}[!ht]
\caption{Average execution times (ms) with $\Tasks=10,000$.}
\resizebox{\columnwidth}{!}{
\begin{tabular}{r|rrrrr}
\multicolumn{1}{l}{\textit{\textbf{}}} & \multicolumn{1}{c}{\textbf{Fed-LBAP}} & \multicolumn{1}{c}{\textbf{OLAR}} & \multicolumn{1}{c}{\textbf{Random}} & \multicolumn{1}{c}{\textbf{Proportional}} & \multicolumn{1}{c}{\textbf{FedAvg}} \\
\midrule
\textit{$100$ resources}                 & $54.134$                                & $18.953$                            & $0.163$                               & $0.023$                                     & $0.002$                               \\
\textit{$1,000$ resources}                & $702.978$                               & $21.318$                            & $1.459$                               & $0.036$                                     & $0.002$                              
\end{tabular}
}
\label{tab2}
\end{table}

Between the two algorithms that generate optimal assignments, we can notice that OLAR performs better than Fed-LBAP in all configurations.
We confirmed their difference using the Mann-Whitney U test with $5\%$ confidence (i.e., all comparisons rejected H0 with p-values $<0.05$, meaning that they come from different distributions).
This non-parametric test was chosen because some of the sampled results did not come from normal distributions (Kolmogorov-Smirnov tests with p-values $<0.05$).

We can also notice that the execution times of OLAR and Fed-LBAP grow at different rates.
When increasing the number of tasks (Fig.~\ref{fig:s2-fixed-resources} and Table~\ref{tab1}), OLAR's execution time increases by less than $2$~ms for every $1,000$ tasks, while Fed-LBAP increases by about $5$~ms.
On the other hand, when increasing the number of resources (Fig.~\ref{fig:s2-fixed-tasks} and Table~\ref{tab2}), OLAR's execution time goes from about $19$ to $21$~ms, while Fed-LBAP's time increases more than tenfold.
This is a natural result of OLAR's complexity ($\Theta(\ResNum+\Tasks\log \ResNum)$).
Finally, OLAR's execution time of up to tens of milliseconds is negligible when compared to the benefit of optimal scheduling for Federated Learning, as a round can easily take over tens of seconds~\cite{wang2020efficient}.

\subsection{Scenario 3: (achieved makespan, with limits)}
\label{subsec:s3}

With OLAR being the only scheduling algorithm that considers the lower and upper limits of resources, we had to extend the other algorithms from the state of the art for this comparison.
We refer to these extended versions by adding the prefix \textit{Ext-} to the original names of the algorithms.

\textbf{Ext-Fed-LBAP} supports limits by sorting only the costs for valid numbers of tasks --- i.e., for each resource $i \in \Resources$, only the costs for tasks in the interval $(\Lower{i},\Upper{i}]$ are considered.
This results in no change to the overall complexity or behavior of the algorithm.

Conversely, \textbf{Ext-FedAvg} and \textbf{Ext-Proportional} required bigger changes to their base algorithms, as solutions that would try to iteratively fix any invalid assignments could lead to $O(\ResNum^2)$ operations for pathological cases.
A solution in two steps was found for these algorithms.
First, we apply their base algorithms and check if their assignments are valid.
If that is the case, their assignments are returned and only $O(\ResNum)$ operations are required.
However, if the assignment is invalid, the second step follows an algorithm similar to OLAR, where the only difference are the kinds of value inserted in the minimum heap:
Ext-FedAvg uses the number of tasks missing to achieve the mean number of tasks per resource $\bar{\Tasks} \coloneqq \floor*{\frac{\Tasks}{\ResNum}}$ (i.e., $\Mapping{i}-\bar{\Tasks}$),
while Ext-Proportional($k$) uses the estimated cost for the next task using its linear regression (i.e., $(\Mapping{i}+1)\times \frac{\Cost{i}{k}}{k}$).
This makes both algorithms $\Theta(\ResNum+\Tasks\log \ResNum)$.


In order to inspect if the addition of limits would lead to any changes to the behavior of OLAR and the extended algorithms, we adapted a subset of the experiments from Scenario~1.
In this scenario, the algorithms had to schedule from $1,000$ to $10,000$ tasks (in increments of $100$) over $100$ resources. 
We only organized the heterogeneous resources in two group: one with all Linear costs, and one with all Quadratic costs.
For the upper and lower limits, we use the rules in Eqs.~\ref{eq:exp:low} and~\ref{eq:exp:up}, as they would lead to invalid assignments for the original scheduling algorithms.

\begin{equation}
\Lower{i} =
  \begin{cases}
    \frac{\bar{\Tasks}}{4} & \text{if~} i = \argmax\limits_{r \in \Resources} \Cost{r}{\Tasks} \\
    4 & \text{otherwise}
  \end{cases}
\label{eq:exp:low}
\end{equation}

\begin{equation}
\Upper{i} =
  \begin{cases}
    \frac{\bar{\Tasks}}{2} & \text{if~} i = \argmin\limits_{r \in \Resources} \Cost{r}{\Tasks} \\
    2\bar{\Tasks} & \text{otherwise}
  \end{cases}
\label{eq:exp:up}
\end{equation}


Fig.~\ref{fig:s3} presents the makespan achieved by OLAR and the extended algorithms in this scenario.
We can notice that these results follow the same general behavior of their counterparts without lower and upper limits in Scenario~1 (Section~\ref{subsec:s1}, Figs.~\ref{fig:s1-linear-100} and~\ref{fig:s1-quadratic-100}).
In all cases, OLAR and Ext-Fed-LBAP find optimal solutions, while none of the other schedulers ever found a single optimal assignment.
Ext-Proportional still found near-optimal solutions for the resources with Linear costs, but its assignments for the Quadratic costs diverges from the optimal.

\begin{figure}[!htb]
    \centering 
    \begin{subfigure}{0.48\columnwidth}
      \includegraphics[width=\columnwidth]{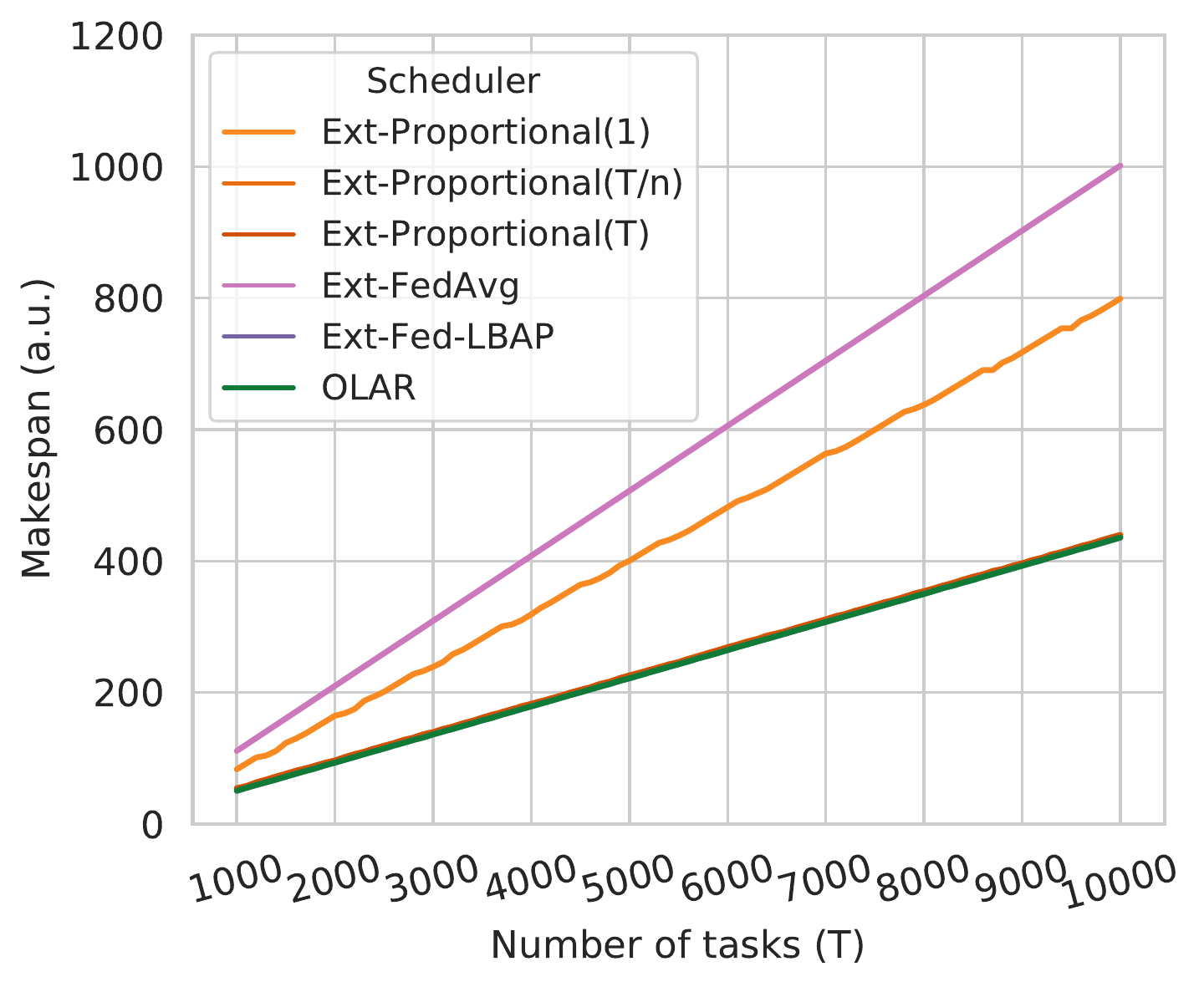}
      \caption{Linear costs, $\ResNum=100$.}
      \label{fig:s3-linear}
    \end{subfigure}\hfill 
    \begin{subfigure}{0.48\columnwidth}
      \includegraphics[width=\columnwidth]{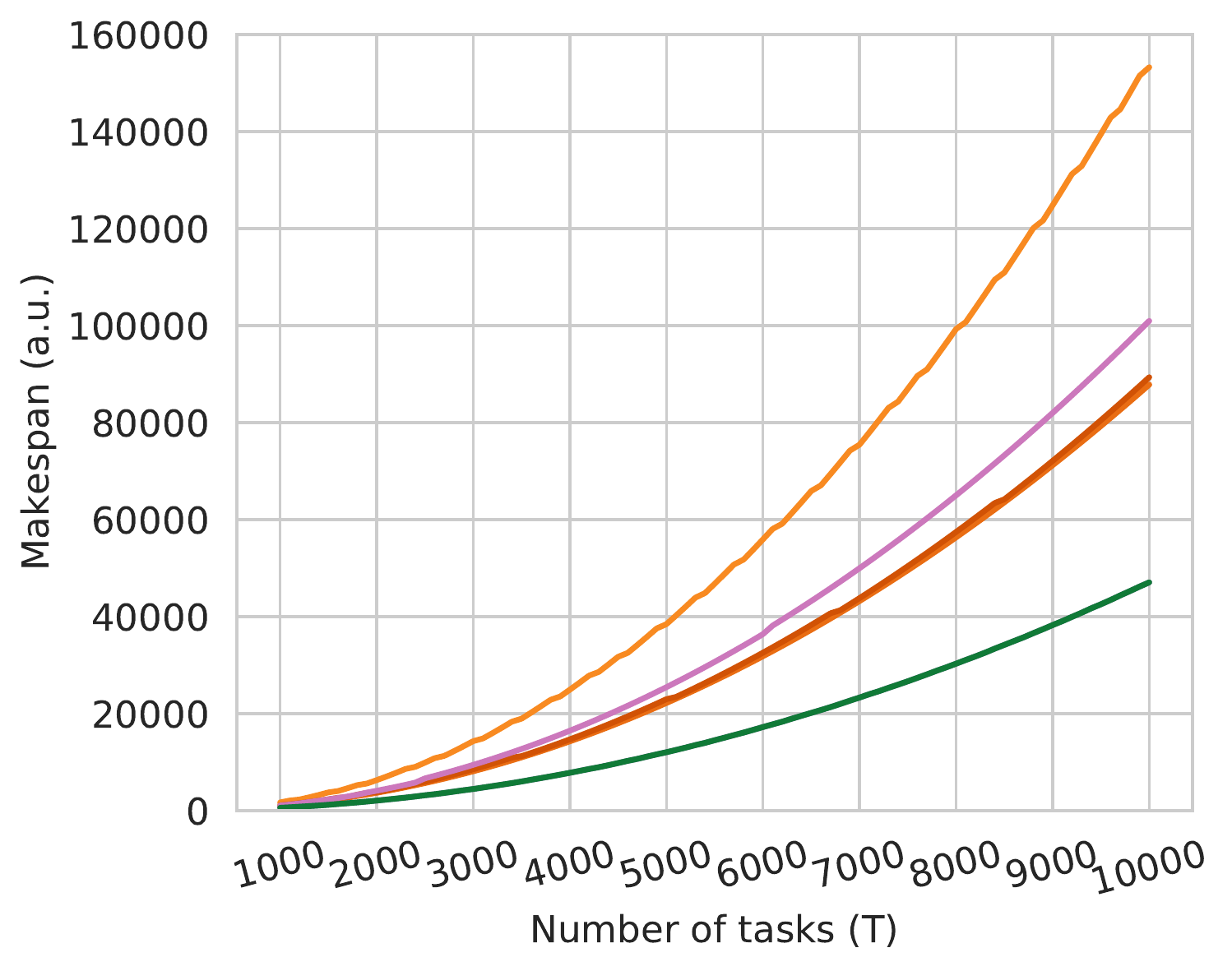}
      \caption{Quadratic costs, $\ResNum=100$.}
      \label{fig:s3-quadratic}
    \end{subfigure}
    \caption{Makespan results for Scenario~3.}
    \label{fig:s3}
\end{figure}

The only major difference seen in these results happened for Ext-Proportional($\floor*{\frac{\Tasks}{\ResNum}}$) and Ext-Proportional($\Tasks$) for Quadratic costs, as they were able to outperform Ext-FedAvg while their original counterparts were not able to do so.
In an additional investigation, we have found that the presence of lower and upper limits restricted the negative effects of the incorrect cost predictions made by Proportional (i.e., a linear regression for a quadratic function), leading to smaller makespans.
Nevertheless, this is not a rule, as the same behavior was not seen when scheduling tasks over $10$ resources only.

\subsection{Scenario 4: (scheduling time, with limits)}
\label{subsec:s4}

After verifying that the extended algorithms compute valid solutions and that their behavior is still similar to their original versions, we move our attention to the execution time of the schedulers in the presence of lower and upper limits of tasks per resource.
The experiments in this scenario follow the same steps of the experiments described in Section~\ref{subsec:s2} with the addition of the lower and upper limits described in Section~\ref{subsec:s3}.
The average execution times measured in this scenario are presented in Fig.~\ref{fig:s4}.
The vertical axis represents the execution time for each scheduler (in ms), while the horizontal axis represents the number of tasks and the number of resources in
Figs.~\ref{fig:s4-fixed-resources} and~\ref{fig:s4-fixed-tasks}, respectively.
Each scheduler is represented by a line connecting their execution times achieved for consecutive cases.
The average times for each scheduler for the smallest and largest cases are also listed in Tables~\ref{tab3} and~\ref{tab4}.

\begin{figure}[b]
    \centering 
    \begin{subfigure}{0.48\columnwidth}
      \includegraphics[width=\columnwidth]{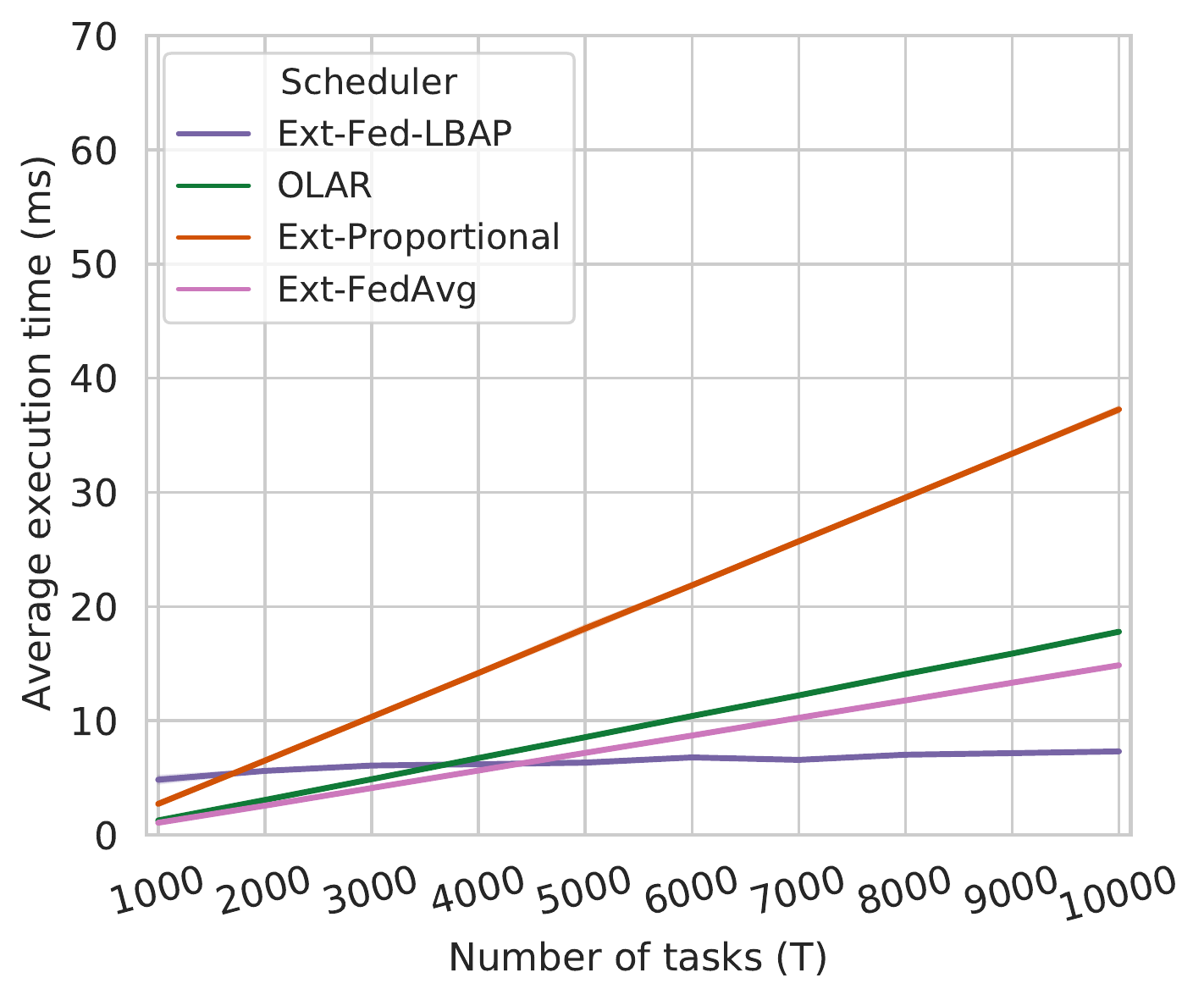}
      \caption{Fixed $\ResNum=100$.}
      \label{fig:s4-fixed-resources}
    \end{subfigure}\hfill 
    \begin{subfigure}{0.48\columnwidth}
      \includegraphics[width=\columnwidth]{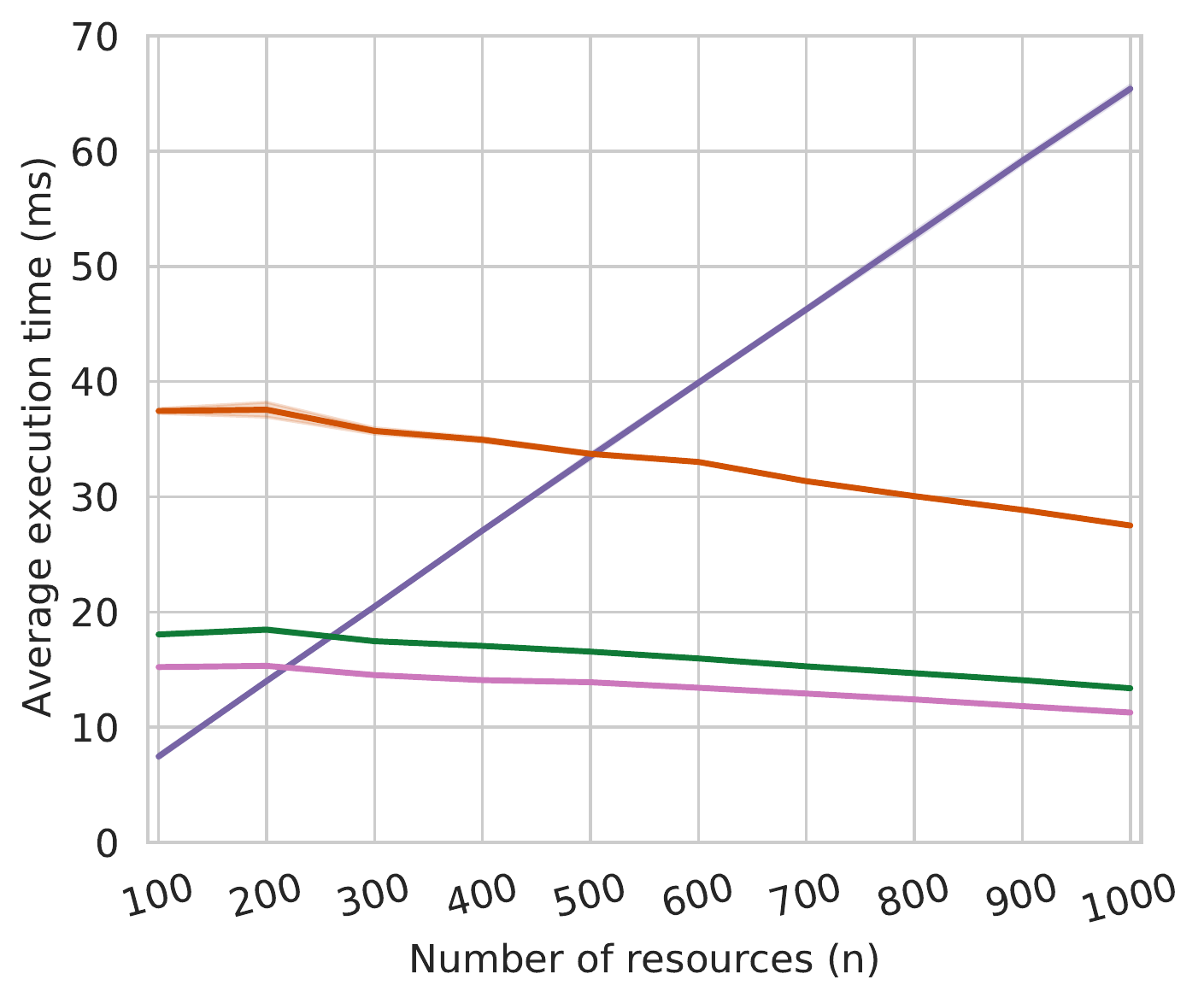}
      \caption{Fixed $\Tasks=10,000$.}
      \label{fig:s4-fixed-tasks}
    \end{subfigure}
     \caption{Scheduling algorithm times for Scenario~4.}
    \label{fig:s4}
\end{figure}

The first major change that can be observed in these results (in comparison to the results in Section~\ref{subsec:s2}) is that the execution times of Ext-Proportional and Ext-FedAvg are now closer to OLAR and Ext-Fed-LBAP.
This is a natural effect of their change in complexity from $O(\ResNum)$ to $O(\ResNum+\Tasks\log \ResNum)$ with the extension.
We can also see that Ext-Proportional seems to take longer than OLAR and Ext-FedAvg to compute its schedule, which is an effect of the kinds of values and operations each uses in their minimum heap.
Nevertheless, the three follow the same general behavior in all cases.

\begin{table}[t]
\caption{Average execution times (ms) with $\ResNum=100$.}
\resizebox{\columnwidth}{!}{
\begin{tabular}{r|rrrr}
\multicolumn{1}{l}{\textit{\textbf{}}} & \multicolumn{1}{c}{\textbf{Ext-Fed-LBAP}} & \multicolumn{1}{c}{\textbf{OLAR}} & \multicolumn{1}{c}{\textbf{Ext-Proportional}} & \multicolumn{1}{c}{\textbf{Ext-FedAvg}} \\
\midrule
\textit{$1,000$ tasks}                    & $4.827$                                     & $1.258$                             & $2.708$                                         & $1.056$                                   \\
\textit{$10,000$ tasks}                   & $7.304$                                     & $17.786$                            & $37.259$                                        & $14.851$
\end{tabular}
}
\label{tab3}
\end{table}

\begin{table}[t]
\caption{Average execution times (ms) with $\Tasks=10,000$.}
\resizebox{\columnwidth}{!}{
\begin{tabular}{r|rrrr}
\multicolumn{1}{l}{\textit{\textbf{}}} & \multicolumn{1}{c}{\textbf{Ext-Fed-LBAP}} & \multicolumn{1}{c}{\textbf{OLAR}} & \multicolumn{1}{c}{\textbf{Ext-Proportional}} & \multicolumn{1}{c}{\textbf{Ext-FedAvg}} \\
\midrule
\textit{$100$ resources}                 & $7.448$                                     & $18.044$                            & $37.448$                                        & $15.211$                                  \\
\textit{$1,000$ resources}                & $65.420$                                    & $13.373$                            & $27.503$                                        & $11.266$
\end{tabular}
}
\label{tab4}
\end{table}

OLAR and Ext-Fed-LBAP show new behaviors in Figs.~\ref{fig:s4-fixed-resources} and~\ref{fig:s4-fixed-tasks}.
In the first figure, we can see that OLAR's execution time grows almost linearly, while Ext-Fed-LBAP's time stays almost constant.
This leads to a situation where OLAR performs better up to $3,000$ tasks, while Ext-Fed-LBAP performs better for higher numbers of tasks (their difference was confirmed with a Mann-Whitney U test with $5\%$~confidence).
This difference comes from the use of upper limits based on the mean number of tasks per resource $\bar{\Tasks}$ (Eq.~\eqref{eq:exp:up}).
While this has no major advantage for OLAR, Ext-Fed-LBAP sees a reduction in the size of its cost matrix.
It is reduced from $\ResNum\times \Tasks$ to $\UpV \approx \ResNum\frac{\bar{\Tasks}}{2} = \ResNum\frac{\Tasks}{2\ResNum} = \frac{\Tasks}{2}$.
This affects the number of operations for sorting the matrix and for the binary search, leading to the performance we see.

In the second case (Fig~\ref{fig:s4-fixed-tasks}), we can see that Ext-Fed-LBAP's execution time grows almost linearly, while OLAR's time actually decreases with the growing number of resources.
This leads to a situation where Ext-Fed-LBAP performs the best up to $200$ resources, and then its execution time surpasses OLAR (their difference was confirmed again with a Mann-Whitney U test with $5\%$~confidence).
This, in turn, comes from the use of lower limits based on a fixed number of tasks (Eq.~\eqref{eq:exp:low}).
This has only a marginal effect for Ext-Fed-LBAP, but OLAR, Ext-Proportional, and Ext-FedAvg see a reduction in the number of iterations on their second loop.
This reduction is based on $\LowV \approx 4\ResNum$, meaning that the loop runs only for $\Tasks-4\ResNum$ iterations.
We have confirmed that OLAR's reduction in execution time is related to the use of lower limits by running an additional version of this experiment with no limits, only lower or upper limits, and using both limits.
The average execution times for $100$ and $1,000$ resources are summarized in Table~\ref{tab5}. It shows that the only factor influencing these execution times is the use of lower limits.

\begin{table}[!ht]
\caption{Average execution times (ms) for OLAR with $\Tasks=10,000$ and variations in lower and upper limits.}
\resizebox{\columnwidth}{!}{
\begin{tabular}{r|cccc}
\multicolumn{1}{l}{\textit{\textbf{}}} & \multicolumn{1}{c}{\textbf{No limits}} & \multicolumn{1}{c}{\textbf{Upper limits}} & \multicolumn{1}{c}{\textbf{Lower limits}} & \multicolumn{1}{c}{\textbf{Both limits}} \\
\midrule
\textit{$100$ resources}               & $19.438$                                 & $19.303$                                    & $18.600$                                    & $18.498$                                   \\
\textit{$1,000$ resources}                & $21.674$                                 & $21.973$                                    & $13.620$                                    & $13.664$                                  
\end{tabular}
}
\label{tab5}
\end{table}

Based on the results from Scenarios~2 and~4,
we can conclude that the presence of lower and upper limits have a beneficial impact on the performance of OLAR and Fed-LBAP.
Nevertheless, the adaptation of $O(\ResNum)$ heuristics to handle limits removes any benefits that their small execution times has brought before.
In this sense, in a situation where costs are known (or properly estimated) and limits should be respected, there would be no reason to use an algorithm to schedule tasks other than the proved optimal OLAR.

\section{Conclusion and Future Work}
\label{sec:conc}

In this paper, we considered the problem of scheduling data to devices in order to minimize the duration of rounds for Federated Learning~(FL).
This problem is very current and relevant, with FL systems growing in scale, thus becoming more easily affected by slowdowns due to stragglers.
We have modeled this problem as a makespan minimization problem with identical, independent, and atomic tasks that have to be assigned to heterogeneous resources with non-decreasing cost functions while respecting lower and upper limits of tasks per resource.
We have proposed a solution to this problem with OLAR (\textit{OptimaL Assignment of tasks to Resources}), a $\Theta(\ResNum+\Tasks\log\ResNum)$ algorithm (for $\Tasks$ tasks and $\ResNum$ resources), which we also proved to be optimal.

In an extensive experimental evaluation (including other algorithms from the state of the art and new extensions to them), we have shown that OLAR can compute optimal assignments of $10,000$ tasks to $1,000$ resources in tens of milliseconds, which is negligible when compared to training rounds that take tens of seconds~\cite{wang2020efficient}.
We have also found that OLAR outperforms Fed-LBAP~\cite{wang2020optimize} and its extended version in most scenarios.
While a Proportional scheduling algorithm has been found to achieve close to optimal assignments when resources follow mostly linear cost functions, our experiments have demonstrated that the execution time benefits of such a heuristic disappear when lower and upper limits of tasks per resource have to be taken into account.
Overall, we can conclude that, given the opportunity, OLAR would be the preferred algorithm for this scheduling problem.


We see some possibilities that would require further study in the future.
First, we would like to conduct experiments in a real FL platform using mobile devices to see how beneficial OLAR can be.
We also plan to investigate ways to adapt OLAR to other scheduling problems in FL that include objectives such as energy consumption reduction and convergence acceleration.
Finally, seeing that Ext-Fed-LBAP also found optimal schedules, and that it had execution times smaller than OLAR for some proportions of tasks per resource when upper limits were present, we plan to try to find a proof of optimality for this algorithm, and to study a way to combine it with OLAR in order to always produce the shortest algorithm's execution times with optimal solutions.

\section*{Acknowledgment}

The author would like to thank Dr. Francieli Z. Boito and Dr. St\'efano D.K. M\'or for their feedback on this manuscript.

\bibliographystyle{IEEEtran}
\bibliography{llpilla}

\appendices
\section{Experimental Evaluation Details for Reproducibility}
\label{appendix:rep}

\textbf{Hardware}: experiments were executed on a Dell Latitude 7390 notebook with an Intel Core i5~8250U processor, 16GB of DDR4 RAM (1$\times$16GB, 2400MHz), and an Intel PCIe SSD with 512GB of capacity.
The computer was plugged to a power source at all times, and the processor was set to run at maximum clock frequency and Turbo Boost.

\textbf{Software}: the computer runs Ubuntu 18.04.5 LTS (kernel version 4.15.0-117-generic). We used Python 3.6.9 with \texttt{numpy} version 1.19.2 for the experiments. Modules \texttt{matplotlib} (3.3.1), \texttt{pandas} (1.1.2), \texttt{seaborn} (0.11.0), and \texttt{scipy} (1.5.1) were used for the visualization and statistical analysis of results.
During timing experiments, no other applications were open besides a terminal.

\textbf{Random number generator (RNG) seeds}: RNG seeds were used to control the generation of parameters for the cost functions of the resources, for the scheduling decisions of the Random scheduler, and for shuffling the order of experiments in the Scenarios~2 and~4.
The seeds for the Random scheduler and for shuffling are increased by one for each new set of inputs.
Seeds were set using \texttt{np.random.seed}. 
Here is a list of the values used per experiment.

\textit{Scenario 1}.
\begin{itemize}
\item Recursive costs: resources: $[0,99]$, 
Random scheduler ($\spadesuit:1000, \clubsuit:2000, \varheartsuit:3000$); 
\item Linear costs: resources: $[100,199]$, 
Random scheduler ($\spadesuit:4000, \clubsuit:5000, \varheartsuit:6000$);
\item Nlogn costs: resources: $[200,299]$, 
Random scheduler ($\spadesuit:7000, \clubsuit:8000, \varheartsuit:9000$);
\item Quadratic costs: resources: $[300,399]$, 
Random scheduler ($\spadesuit:10000, \clubsuit:11000, \varheartsuit:12000$);
\item Mixed costs: resources: $[400,499]$, 
Random scheduler ($\spadesuit:13000, \clubsuit:14000, \varheartsuit:15000$).
\end{itemize}

\textit{Scenario 2}.
\begin{itemize}
\item Fixed resources: resources: $[0,99]$, Random scheduler: $1000$, Shuffle: $0$;
\item Fixed tasks: resources: $[0,999]$, Random scheduler: $1000$, Shuffle: $1000$.
\end{itemize}

\textit{Scenario 3}. 
\begin{itemize}
\item Linear costs: resources: $[600,699]$.
\item Quadratic costs: resources: $[700,799]$.
\end{itemize}

\textit{Scenario 4}. 
\begin{itemize}
\item Fixed resources: resources: $[0,99]$, Shuffle: $1000$;
\item Fixed tasks: resources: $[0,999]$, Shuffle: $2000$.
\end{itemize}

\textbf{Fed-LBAP}: a series of minor changes had to be made on the algorithm described in~\cite{wang2020optimize} in order for it to converge to a solution. 
They include: 
computing the \textit{median} and $D'$ at the beginning of the loop (not before the loop);
changing the stop condition of the main loop (\textit{min} is always smaller than \textit{max} when using the \textit{floor} operator for the median);
checking if $D' \ge D$ when the algorithm stops, or else recomputing the solution using \textit{median}$\leftarrow$\textit{max}.
For extended-Fed-LBAP, we also remove tasks from resources if the algorithm assigns more tasks than requested.
This happens in specific cases where multiple resources have costs equal to $\Makespan$, or when the lower limit of a resource sets the makespan.
For instance, if $\Resources = \{0,1\}$, $\Tasks=3$, $\Costs{0}{0,0.5, 1, 1.5}$, $\Costs{1}{0, 0.7, 1, 1.3}$, Fed-LBAP would assign 4~tasks ($\Mapping{0}=2, \Mapping{1}=2$) for a $\Makespan=1$.

\balance

\end{document}